\documentclass[]{article}

\usepackage{fullpage}
\usepackage{mathtools}
\usepackage{amssymb}
\usepackage{amsthm}
\usepackage{bm}
\usepackage{mleftright}
\usepackage[hypertexnames=false]{hyperref}
\usepackage[capitalize,noabbrev]{cleveref}
\usepackage{xcolor}
\usepackage[shortlabels]{enumitem}
\usepackage{algorithm}
\usepackage{algpseudocodex}
\usepackage{authblk}
\usepackage{etoolbox}
\usepackage[page]{appendix}
\usepackage[short]{optidef}
\usepackage{lmodern}
\usepackage[T1]{fontenc}
\usepackage[natbib=true]{biblatex}
\usepackage[final]{microtype}

\makeatletter
\theoremstyle{plain}
\newtheorem{thm}{Theorem}
\newtheorem{lem}{Lemma}
\newtheorem{assume}{Assumption}

\DeclarePairedDelimiter{\abs}{\lvert}{\rvert}
\DeclarePairedDelimiter{\norm}{\lVert}{\rVert}
\DeclarePairedDelimiterXPP{\inner}[2]{}{\langle}{\rangle}{}{#1,#2}

\newcommand{\bI}{\mathbb{I}}
\newcommand{\bN}{\mathbb{N}}
\newcommand{\bP}{\mathbb{P}}
\newcommand{\bR}{\mathbb{R}}

\newcommand{\cA}{\mathcal{A}}
\newcommand{\cB}{\mathcal{B}}
\newcommand{\cC}{\mathcal{C}}
\newcommand{\cL}{\mathcal{L}}
\newcommand{\cM}{\mathcal{M}}
\newcommand{\cN}{\mathcal{N}}
\newcommand{\cV}{\mathcal{V}}
\newcommand{\vone}{\bm{1}}
\newcommand{\vzero}{\bm{0}}
\newcommand{\PMan}{\widetilde{\cM}}

\DeclareMathOperator*{\argmin}{argmin}
\DeclareMathOperator{\rT}{T}
\DeclareMathOperator{\Diff}{D}
\DeclareMathOperator{\grad}{grad}
\DeclareMathOperator{\Hess}{Hess}
\DeclareMathOperator{\tr}{tr}
\DeclareMathOperator{\rank}{rank}
\DeclareMathOperator{\svect}{svec}
\DeclareMathOperator{\smat}{smat}
\DeclareMathOperator{\vect}{vec}
\DeclareMathOperator{\diag}{diag}
\DeclareMathOperator{\dvec}{dvec}
\DeclareMathOperator{\Retr}{Retr}
\DeclareMathOperator{\Proj}{Proj}
\DeclareMathOperator{\poly}{poly}%
\DeclareMathOperator{\Orth}{Orth}%

\newcommand{\printfnsymbol}[1]{%
	\textsuperscript{\@fnsymbol{#1}}%
}

\AtBeginEnvironment{appendices}{\crefalias{section}{appendix}\crefalias{subsection}{appendix}}
\makeatother

\title{Scalable Second-order Riemannian Optimization for $K$-means Clustering}
\author[1]{Peng Xu\thanks{These authors contributed equally.}}
\author[2]{Chun-Ying Hou\printfnsymbol{1}}
\author[3]{Xiaohui Chen}
\author[2]{Richard Y.~Zhang}
\affil[1]{Department of Statistics, University of Illinois Urbana-Champaign \authorcr\texttt{pengxu1@illinois.edu}}
\affil[2]{Department of Electrical and Computer Engineering, University of Illinois Urbana-Champaign \authorcr\texttt{\{cyhou2,ryz\}@illinois.edu}}
\affil[3]{Department of Mathematics, University of Southern California \authorcr\texttt{xiaohuic@usc.edu}}

\begin{document}

\maketitle

\begin{abstract}
Clustering is a hard discrete optimization problem. Nonconvex approaches such as low-rank semidefinite programming (SDP) have recently demonstrated promising statistical and local algorithmic guarantees for cluster recovery. Due to the combinatorial structure of the $K$-means clustering problem, current relaxation algorithms struggle to balance their constraint feasibility and objective optimality, presenting tremendous challenges in computing the second-order critical points with rigorous guarantees. In this paper, we provide a new formulation of the $K$-means problem as a smooth unconstrained optimization over a submanifold and characterize its Riemannian structures to allow it to be solved using a second-order cubic-regularized Riemannian Newton algorithm. By factorizing the $K$-means manifold into a product manifold, we show how each Newton subproblem can be solved in linear time. Our numerical experiments show that the proposed method converges significantly faster than the state-of-the-art first-order nonnegative low-rank factorization method, while achieving similarly optimal statistical accuracy.
\end{abstract}

\section{Introduction}\label{sec:intro}
Clustering is a cornerstone of modern unsupervised learning, where the goal is to group similar observations into meaningful clusters. The problem is commonly approached through the $K$-means formulation, which seeks to partition $n$ data points $X_{1},X_{2},\dotsc,X_{n}\in\bR^{d}$ into $K$ disjoint groups $G_{1},\dotsc,G_{K}$ by maximizing the total intra-cluster similarity:
\begin{equation}\label{eqn:Kmeans}
\max_{G_{1},\dots,G_{K}} \mleft\{ \sum_{k=1}^{K}\frac{1}{\abs{G_{k}}}\sum_{i,j\in G_{k}}\inner{X_{i}}{X_{j}}: \bigsqcup_{k=1}^{K}G_{k}=[n]\mright\}.
\end{equation}
Here, the inner product $\inner{X_{i}}{X_{j}}=X_{i}^\top X_{j}$ is used to measure pairwise similarity, $\abs{G_{k}}$ denotes the cardinality
of $G_{k}$, and $\sqcup$ denotes disjoint union. Most common algorithms for $K$-means clustering, including Lloyd's algorithm~\citep{Lloyd1982_TIT} and spectral clustering~\citep{NgJordanWeiss2001_NIPS,vanLuxburg2007_spectralclustering}, can be understood as heuristics for finding ``good enough'' solutions to the discrete optimization (\ref{eqn:Kmeans}). These methods do not come with any guarantees of local optimality, let alone global optimality. Indeed, it is commonly argued that globally solving (\ref{eqn:Kmeans}) is NP-hard in the worst-case~\citep{Dasgupta2007,aloise2009np}, and would lead to statistically meaningless clustering that overfits the data. 

Yet in average-case regimes, globally solving the $K$-means optimization problem (\ref{eqn:Kmeans}) can be both computationally tractable as well as statistically optimal. In particular, when the data $X_{1},\dots,X_{n}$
arise from a Gaussian mixture model with sufficiently well-separated components, \citet{chen2021cutoff} showed that a well-known semidefinite programming (SDP) relaxation of \citet{PengWei2007_SIAMJOPTIM}, written
\begin{equation}\label{eq:sdp}
\max_{Z\in\bR^{n\times n}}\Biggl\{ \inner{XX^\top }Z+\mu\sum_{i,j}\log(Z_{i,j})_{+} : Z\vone_{n}=\vone_{n}, \,\tr(Z)=K,\, Z\succeq0 \Biggr\}, 
\end{equation}
where $X=[X_{1},\dotsc,X_{n}]^\top$ and $\log(Z_{i,j})_{+}\coloneqq\log(\max\{Z_{i,j},0\})$, is guaranteed to compute the globally
optimal clusters $G_{1}^{\star},\dotsc,G_{K}^{\star}$ for (\ref{eqn:Kmeans})
in the limit $\mu\to0^{+}$ in polynomial time, that in turn recover
the ground truth partitions. Note that the formulation (\ref{eq:sdp}) is equivalent to the standard $K$-means SDP formulation with the elementwise nonnegativity constraint $Z_{i,j} \geq 0$ in~\citet{PengWei2007_SIAMJOPTIM,chen2021cutoff} (see Appendix~\ref{app:SDP_formulation_equiv} for more discussions). Moreover, this recovery occurs as soon
as the separation between the clusters is large enough for it to be possible. Put in another way, if solving (\ref{eq:sdp}) does not recover the ground truth partitions, then the clusters are too closely spaced in a way that makes recovery inherently impossible in an information-theoretic
limit sense, see \Cref{subsec:background_SDP_Kmeans} for more details.

Unfortunately, the SDP (\ref{eq:sdp}) is not a practical means of solving
(\ref{eqn:Kmeans}) to global optimality, due to its need to optimize
over an $n\times n$ matrix to cluster $n$ samples. Following \citet{BurerMonteiro2003,boumal2020deterministic}, a natural alternative
is to factor $Z=UU^{\top}$ into its $n\times r$ factor matrix $U$
for rank parameter $r\ge K$, impose the logarithmic penalty over
$U$ instead of $Z$, and then directly optimize over $U$:
\begin{equation}\label{eq:manif}
\max_{U\in\bR^{n\times r}}\Biggl\{\inner{XX^\top}{UU^\top}+\mu\sum_{i,j}\log(U_{i,j})_{+}
: UU^\top\vone_n=\vone_n,\,\tr(UU^\top)=K\Biggr\}.
\end{equation}
This reduces the number of variables and constraints from $O(n^{2})$
down to $O(n)$, but at the cost of giving up the convexity of the
SDP. In general, we can at best hope to compute critical points, which
may be spurious local minima or saddle points. The core motivation
for our approach, and the impetus for this paper, is the surprising
empirical observation that all second-order critical points are global
optima in this setting; this is formalized as the following assumption.

\begin{assume}[Benign nonconvexity]\label{asm:bengin}In the average-case
regime when (\ref{eq:sdp}) globally solves (\ref{eqn:Kmeans}), all
approximate second-order critical points in (\ref{eq:manif}) are
within a neighborhood of a global optimum.\end{assume}

The phenomenon of benign nonconvexity is well-documented in
the \emph{unconstrained} version---optimizing over
semidefinite $Z\succeq0$ by factorizing $Z=UU^{\top}$---dating back to the early works of \citet{BurerMonteiro2003}. In contrast, it is rarely seen in our \emph{nonnegative} variant, which adds the elementwise constraint $U \ge 0$ to enforce doubly nonnegativity in $Z=UU^{\top}$.  Despite a superficial similarity, the two formulations differ in fundamental ways, with the nonnegative case known to admit numerous spurious critical points; see \Cref{sec:related} for some classic and recent examples. Nevertheless, we consistently observe that all second-order critical points correspond to global optima, that in turn successfully recover the
optimal clusters.

\subsection{Contributions: Cheap and fast convergence to second-order critical points}

Under Assumption~\ref{asm:bengin}, globally solving the $K$-means optimization problem (\ref{eqn:Kmeans}) reduces to that of computing a second-order critical point for (\ref{eq:manif}). Unfortunately,
in the constrained nonconvex setting, there is no general-purpose algorithm that is rigorously guaranteed to compute a second-order
critical point.
The core issue is the need to maintain \emph{feasibility}, i.e.\@ for each iterate $U$ to satisfy the nonconvex constraints $UU^\top\vone_n=\vone_n$ and $\tr(UU^\top)=K$, while making progress towards optimality. General-purpose solvers like \texttt{fmincon} \citep{byrd2000trust} and \texttt{knitro} 
\citep{byrd2006k} promise
convergence only to critical points of an underlying merit function,
which may be infeasible for the original problem. Augmented Lagrangian
methods guarantee convergence only to first-order critical points,
and only when starting within a local neighborhood~\citep{ZhuangChenYangZhang2024_ICLR}.
This is a significant departure from the unconstrained nonconvex setting,
where a diverse range of algorithms---both cheap first-order algorithms
like gradient descent, as well as rapidly-converging second-order
methods like trust-region Newton's method---globally converge to
a second-order critical point starting from any initial point.

Our first contribution is to present an interpretation of (\ref{eq:manif})
as a \emph{smooth unconstrained optimization over a Riemannian manifold}. This allows the immediate
benefit of extending the wide array of unconstrained optimization
algorithms to the constrained setting, as well as their accompanying
guarantees for first- and second-order optimality. For the first
time in the context of $K$-means, we open the possibility to guarantee global convergence to first- and second-order
optimality. 

Our second contribution is to show that \emph{second-order Riemannian
algorithms can be implemented with linear per-iteration costs} with
respect to the number of samples $n$. In other words, of all practical
algorithms to compute second-order critical points, we show that the
one with the best iteration complexity (second-order methods) can
be improved to have the same per-iteration costs as first-order methods. Our final algorithm computes $\epsilon$ second-order points in $n\cdot\epsilon^{-3/2}\cdot\poly(r,d)$
time. 

\subsection{Related work}\label{sec:related}

Benign nonconvexity in the unconstrained Burer--Monteiro factorization
$Z=UU^{\top}$ has been empirically observed since the early 2000s~\citep{BurerMonteiro2003}, and widely exploited in nonconvex low-rank algorithms in machine learning. In the past decade, theory has been developed to explain this phenomenon under some specialized settings~\citep{bhojanapalli2016global,ge2016matrix,bandeira2016low,boumal2016non,ge2017no}. Unfortunately, these guarantees tend to be conservative in the number of samples or the level of noise; they capture the general phenomenon but cannot
rigorously explain what is broadly observed in practice.

In contrast, the \emph{nonnegative} Burer--Monteiro factorization $Z=UU^{\top}$
with $U\ge0$ is widely understood \emph{not} to exhibit benign nonconvexity.
To give two simple examples, the functions $f(U)=\inner{SU}U$ and
$f(U)=\norm{UU^{\top}-U_{\star}U_{\star}^{\top}}^{2}_F$ are easily confirmed
to exhibit benign nonconvexity over $U\in\bR^{n\times r}$. But imposing $U\ge0$ causes spurious local minima to proliferate; this
is unsurprising because both problems, namely copositive testing~\citep{murty1987some}
and complete positive testing~\citep{dickinson2014computational},
are well-known to be NP-hard. For a more sophisticated example, the
function $f(U)=\norm{\mathcal{A}(UU^{\top}-U_{\star}U_{\star}^{\top})}^{2}$
is well known to exhibit benign nonconvexity when the linear operator
$\mathcal{A}:S^{n}(\bR)\to\bR^{m}$ satisfies the restricted isometry property
(RIP)~\citep{bhojanapalli2016global}. In this context, a recent work~\citep{zhang2025nonnegative} gave a strong counterexample for the equivalent statement over $U\ge0$. 

Therefore, even though $K$-means is widely known to admit a nonnegative
Burer--Monteiro reformulation \citep{PengWei2007_SIAMJOPTIM}, there have been only two prior works that actually follow this approach, to the best of our knowledge.
Neither of these can rigorously guarantee global optimality under
Assumption~\ref{asm:bengin}. The first is the first-order Riemannian
method introduced by \citet{CarsonMixonVillarWard_manifold-Kmeans}.
It solves the following:
\begin{equation}
\min_{U\in\cM'}\mleft\{ -\inner{XX^{\top}}{UU^{\top}}+\lambda\norm*{U_{-}}_{F}^{2}\mright\}\label{eqn:Carson_manifold}
\end{equation}
where \(\cM'\coloneqq\mleft\{U\in\bR^{n\times K}:U^{\top}U=I_{K},UU^{\top}\vone_n=\vone_n\mright\},\) and $U_{-}=\max\{-U,0\}$ is the (entrywise) negative part of $U$, $\lambda\geq0$ is the penalty parameter for $U\geq0$. Although superficially similar, their approach fundamentally lacks a convergence guarantee to a second-order critical point, due to: (i) their nonsmooth objective; (ii) their use of a smooth penalty, which cannot truly enforce feasibility $U\ge0$; (iii) their use of a first-order method, which can get trapped at a saddle point. Moreover, their manifold is geometrically complicated, necessitating an expensive retraction that costs $O(n^{2})$ time, which prevents their method from scaling to large datasets. 

The second work is the nonnegative low-rank (NLR) method of \citet{ZhuangChenYangZhang2024_ICLR}. This is a simple projected gradient descent that directly projects $U$ onto the nonnegative spherical constraint and deals the row sum constraint $UU^{\top}\vone_n=\vone_n$ via the augmented Lagrangian
method. It is a first-order primal-dual method that can only achieve local linear convergence in a neighborhood of its global solution. Like~\citet{CarsonMixonVillarWard_manifold-Kmeans}, it is unclear whether there is a pathway that this algorithm can lead to a global optimality guarantee, or even to second-order optimality.

Recently, hybrid methods \citep{Wang2025, msc2025, doi:10.1287/opre.2024.1137, doi:10.1137/22M1474539} have been proposed for low-rank SDPs. These approaches handle simple manifold constraints via Riemannian optimization, while enforcing the remaining constraints through augmented-Lagrangian updates. In contrast, our reformulated $K$-means problem requires strict feasibility with respect to both nonnegativity and the simplex-type manifold constraints, since the recovered factor must encode a valid partition. Augmented-Lagrangian or projection-based schemes do not preserve this property and would break the structural guarantees on which our method relies.

\section{Background}\label{sec:background}

\subsection{SDP relaxation of \texorpdfstring{$K$}{K}-means}\label{subsec:background_SDP_Kmeans}

Despite the worst-case NP-hardness of the $K$-means clustering optimization problem~(\ref{eqn:Kmeans}), common practical heuristics and relaxed formulations like Lloyd's algorithm~\citep{Lloyd1982_TIT}, spectral clustering~\citep{NgJordanWeiss2001_NIPS,vanLuxburg2007_spectralclustering}, nonnegative matrix factorization (NMF)~\citep{6061964,kuang2015symnmf,wang2012nonnegative}
and SDPs~\citep{PengWei2007_SIAMJOPTIM,Royer2017_NIPS,FeiChen2018,ChenYang_nonEuclideanKmeans} work surprisingly well at solving it for real-world data. To explain
this discrepancy between theory and practice, suppose that the data
$X_{1},\dots,X_{n}\in\bR^{d}$ are generated from a standard Gaussian
mixture model (GMM)
\begin{equation}
X_{i}=\mu_{k}+\varepsilon_{i},\qquad\varepsilon_{i}\overset{\text{i.i.d.}}{\sim}\cN(0,\sigma^2I_{d}),\quad\text{for }i\in G_{k}^{*},\label{eqn:GMM}
\end{equation}
where $G_{k}^{*}$ denotes the ground truth clusters. \citet{chen2021cutoff}
proved that the SDP (\ref{eq:sdp}) of \citet{PengWei2007_SIAMJOPTIM} (as $\mu \to 0^{+}$)
achieves a \emph{sharp phase transition} on the separation of centroids for the clustering problem, in any dimension $d$ and sample size $n$. Let 
\begin{equation}
\overline{\Theta}^{2}\coloneqq4\sigma^{2}\mleft(1+\sqrt{1+\frac{Kd}{n\log{n}}}\mright)\log{n},\label{eqn:recovery_threshold}
\end{equation}
and $\Theta_{\min}\coloneqq\min_{1\leq j<k\leq K}\norm{\mu_{j}-\mu_{k}}$
be the minimum centroid separation. Assume that $m=n/K$ is an integer without loss of generality and consider any $\alpha > 0$. As soon as the exact recovery becomes possible in the regime $\Theta_{\min}\geq(1+\alpha)\overline{\Theta}$, the SDP approach (\ref{eq:sdp}) solves the $K$-means problem without clustering error with high probability. For precise statements on the information-theoretic threshold, please refer to Theorem~\ref{thm:recovery_threshold} in Appendix~\ref{sec:info_threshold}. As an immediate consequence of the global optimality guarantee of the $K$-means SDP in (\ref{eq:sdp}), we deduce that the global solution of the nonconvex low-rank SDP in~(\ref{eq:manif}) solves the $K$-means clustering problem in~(\ref{eqn:Kmeans}) in the exact recovery regime.




Next, from the membership matrix
$Z$, we would like to convert it to the cluster label.
\begin{lem}
\label{lem:membership2clusterlabel} Let $Z=Z^\top\in\bR^{n\times n}$
be the symmetric block-diagonal matrix defined by $Z_{ij}=1/\abs{G_{k}}$
if $i,j\in G_{k}$, and $Z_{ij}=0$ otherwise. Then for any integer
$r\in[K,n]$, there is a unique (up to column permutation) $U\in\bR_{+}^{n\times K}$ such that $Z=U U^\top$. Moreover, $U$ can be recovered
from any $\hat{U}\in\bR^{n\times r}$ satisfying $Z=\hat{U}\hat{U}^\top$ in $n\cdot\poly(r)$
time.
\end{lem}

For each block-diagonal membership matrix $Z$, the unique $U\in\bR_{+}^{n\times K}$
in Lemma~\ref{lem:membership2clusterlabel} is the associated group
assignment matrix, i.e.\ the $k$-th column of $U$ provides a one-hot encoding of membership in the $k$-th cluster.

\subsection{Critical points in constrained optimization}

The problems considered in this paper are instances of the following
\begin{equation}
\min_{U\in\cM}f(U),\qquad\cM=\{U\in\bR^{n\times r}:\cA(UU^{\top})+\cB(U)=c\},\label{eq:manopt}
\end{equation}
where the linear operators $\cA\colon\bR^{n\times n}\to\bR^{m}$ and $\cB\colon\bR^{n\times r}\to\bR^{m}$
and right-hand side $c\in\bR^{m}$ together are assumed to satisfy the linear independence constraint qualification (LICQ)
\begin{equation}
2[\cA^{\top}(y)]U+\cB^{\top}(y)=\vzero\iff y=\vzero\quad\forall U\in\cM.\label{eq:manlicq}
\end{equation}
In this context, $U\in\bR^{n\times r}$ is said to be \emph{feasible}
if it satisfies $U\in\cM$. The feasible point $U$ is an $\epsilon$-first-order
critical point if it satisfies
\begin{equation}
\text{exists}\quad y\in\bR^{m}\quad \text{s.t.}\quad \norm*{\nabla f(U)+2[\cA^{\top}(y)]U+\cB^{\top}(y)}\le\epsilon,\label{eq:FOC}
\end{equation}
and an $\epsilon$-second-order critical point if it additionally satisfies
\begin{gather}
\inner{\nabla^{2}f(U)+2[\cA^{\top}(y)]}{\dot{U}\dot{U}^{\top}}\ge\sqrt{\epsilon}\norm{\dot{U}}^{2}\quad\forall\dot{U}\in\rT_{U}\cM\label{eq:SOC}
\end{gather}
over the \emph{tangent space} of $\cM$ at the point $U$, given by $\rT_{U}\cM=\{\dot{U}\in\bR^{n\times r}:\cA(U\dot{U}^{\top}+\dot{U}U^{\top})+\cB(\dot{U})=0\}$.

Under LICQ (\ref{eq:manlicq}), every local minimum (and hence the global minimum) is guaranteed to be an $\epsilon$-second-order critical point (for any $\epsilon\ge0$). Unfortunately, there is no general-purpose algorithm that is guaranteed to converge to a critical point, due to the need to achieve and maintain feasibility across all iterates. 

\subsection{Second-order Riemannian optimization}\label{subsec:Riemann}

Riemannian algorithms are special algorithms that maintain feasible
iterates through a problem-specific \emph{retraction} operator, and
are hence able to rigorously guarantee convergence to critical points.
The basic idea is to improve a feasible iterate $U\in\cM$ by tracing
a smooth curve on the feasible set $\gamma:[0,\epsilon)\to\cM$ that
begins at $\gamma(0)=U$ and proceeds in a direction of descent $\dot{\gamma}(0)=\dot{U}\in\rT_{U}\cM$.
In analogy with unconstrained algorithms, a good choice of $\dot{U}\in\rT_{U}\cM$
is found through a local Taylor expansion
\begin{equation}
f\bigl(\gamma(t)\bigr)=f(U)+t\inner{\grad f(U)}{\dot{U}}+\frac{t^{2}}{2}\inner{\Hess f(U)[\dot{U}]}{\dot{U}}+O(t^{3}),\label{eq:RTaylor}
\end{equation}
where $\grad f$ and $\Hess f$ are respectively the \emph{Riemannian
gradient} and \emph{Riemannian Hessian} of $f$ on the manifold $\cM$.
Afterwards, we trace the curve $\gamma(t)=\Retr_{U}(t\dot{U})$ using
a \emph{second-order retraction} operator $\Retr_{U}\colon\rT_{U}\cM\to\cM$
satisfying
\[
\Retr_{U}(0)=U,\qquad\mleft.\frac{d}{dt}\Retr_{U}(t\dot{U})\mright|_{t=0}=\dot{U},\qquad\mleft.\frac{d^{2}}{dt^{2}}\Retr_{U}(t\dot{U})\mright|_{t=0}\perp\rT_{U}\cM,
\]
for all $U\in\cM$ and all $\dot{U}\in\rT_{U}\cM$. After choosing
step-size $t$ so that $U_{\text{new}}=\gamma(t)$ makes a sufficient improvement
over $U$, we repeat the algorithm until it reaches an $\epsilon$-second-order
critical point satisfying $\norm{\grad f(U)}\le\epsilon$ and $\lambda_{\min}\bigl(\Hess f(U)\bigr)\ge-\sqrt{\epsilon}$,
which incidentally corresponds exactly to (\ref{eq:FOC}) and (\ref{eq:SOC}).
Proofs for the following convergence result can be found in \citet{zhang2018cubic,boumal2019global,agarwal2021adaptive};
we have chosen the simplest but most restrictive settings to ease
the exposition.

\begin{thm}[Riemannian cubic-regularized Newton]
\label{thm:rcr} Suppose that $\min_{U\in\cM}f(U)>-\infty$, and that
the pullback $\hat{f}=f\circ\Retr_{U}$ has Lipschitz continuous Hessian
for all $U\in\cM$. Then, there exists a sufficiently large regularizer
$L$ such that $U_{k+1}=\Retr_{U_{k}}(\dot{U}_{k})$ where
\begin{align*}
\dot{U}_{k}=\argmin_{\dot{U}\in\rT_{U}\cM}f(U)+\inner{\grad f(U)}{\dot{U}}+\frac{1}{2}\inner{\Hess f(U)[\dot{U}]}{\dot{U}}+\frac{L}{6}\norm{\dot{U}}^{3}
\end{align*}
converges to an $\epsilon$-second order critical point in
$O(\epsilon^{-3/2})$ iterations, independent of dimension.
\end{thm}

Each iteration of Riemannian cubic-regularized Newton solves an expensive Newton subproblem.
Although it converges in far fewer iterations compared to gradient methods, it is practically competitive only when the added cost of solving the Newton subproblem can be offset by the corresponding reduction in iteration count.

\section{Formulation and solution of $K$-means as manifold optimization}\label{sec:formulation}
We now explain how we solve (\ref{eq:manif})
using a Riemannian optimization approach. As a first attempt, we can
indeed verify that the the constraint set in (\ref{eq:manif}), written
\begin{equation}\label{eqn:our_Kmeans_manifold}
    \cM\coloneqq\cM_r=\mleft\{U\in\bR^{n\times r}:UU^\top \vone_n=\vone_n,\,\tr(UU^\top)=K\mright\},
\end{equation}
is a manifold by checking that (\ref{eq:manlicq}) holds (cf. Lemma~\ref{lem:LICQ_submanifold_Kmeans} in the appendix). In fact,
directly applying Riemannian optimization techniques results in a
$K$-means algorithm very similar to the one proposed in \citet{CarsonMixonVillarWard_manifold-Kmeans}.
The immediate and critical difficulty with this approach is the lack
of an efficient retraction operator, which must be called at every
iteration to keep iterates feasible $U\in\cM$. For example, \citet{CarsonMixonVillarWard_manifold-Kmeans}
used a complicated exponential retraction that costs $O(n^{2})$
time, hence bottlenecking the entire algorithm and preventing it from
scaling to large $n$.

Instead, our first contribution in this paper is to reformulate (\ref{eq:manif}) by establishing a submersion from the product manifold $\PMan=\cV\times\Orth(r)$ to $\cM$, where 
\begin{align*}
\cV&=\mleft\{V\in\bR^{n\times(r-1)}:\vone_n^\top V=0,\;\tr(VV^\top )=K-1\mright\}, \\
\Orth(r)&=\mleft\{Q\in\bR^{r\times r}:QQ^\top =I_{r}\mright\}.
\end{align*}
In words, $\cV$ is a projected hypersphere and $\Orth(r)$ is the
set of $r\times r$ orthonormal matrices. 
\begin{thm}\label{prop:factor}
Let $\varphi(V,Q)\coloneqq\hat{V}Q$, where $\hat{V}\coloneqq \begin{bmatrix}\hat{\vone}_n & V\end{bmatrix}$ with $\hat{\vone}_n\coloneqq(1/\sqrt{n})\vone_n$. Then $\cM=\varphi(\PMan)$. Moreover, the Jacobian $\Diff\varphi\colon\rT\PMan\to\rT\cM$ is surjective for all $(V,Q)\in\PMan$, i.e., $\varphi$ is a submersion.
\end{thm}

Having established the submersion property of $\varphi$, it is a
standard result that every $\epsilon$-second order point of $\PMan$
is also an $c\epsilon$-second order point on $\cM$ for some constant
rescaling factor $c$; see e.g.\ Example 3.14 and the surrounding text
in \citet{Levin2025}. Therefore, to solve (\ref{eq:manif}), we equivalently solve
\begin{align}
\min_{(V,Q)\in\PMan} & \inner{C}{VV^\top }-\mu\sum_{i,j}\log\bigl(\varphi_{i,j}(V,Q)\bigr)_{+},\label{eq:manif2}
\end{align}
where $C=-XX^\top $ is the (negative) data Gram matrix, and $\varphi_{i,j}$
is the $(i,j)$-th element of the operator $\varphi$ in \Cref{prop:factor}. 
A basic but critical benefit of the reformulation (\ref{eq:manif2})
is that the product manifold $\PMan$ admits a simple second-order
retraction via its Euclidean projection \citep[Sec.~5.12]{boumal2023intromanifolds}
\[
\Retr_{(V,Q)}(\dot{V},\dot{Q})=\begin{bmatrix}
\Proj_{\cV}(V+\dot{V}) & \Proj_{\Orth(r)}(Q+\dot{Q})
\end{bmatrix},
\]
where \[\Proj_{\cV}(V)=\sqrt{K-1}\frac{V-n^{-1}\vone_n\vone_n^{\top}V}{\norm{V-n^{-1}\vone_n\vone_n^{\top}V}}\quad\text{and}\quad\Proj_{\Orth(r)}(Q)=(QQ^{\top})^{-1/2}Q.\]
It is easy to check that the retraction above costs just $O(nr+r^{3})$ time to evaluate. In \Cref{subsec:deriv}, we give explicit expressions for
the Riemannian gradient and Hessian and explain how they can be computed
in $O(nr+r^{3})$ time. 

The appearance of the logarithmic penalty in (\ref{eq:manif2}) presents two difficulties. First, as a practical concern, any algorithm for(\ref{eq:manif2}) must begin at a \emph{strictly feasible} point $(V_{0},Q_{0})\in\PMan$ that additionally satisfies $\varphi(V_{0},Q_{0})>0$. In \Cref{subsec:feas}, we provide a good strictly feasible initial point, and prove that points exist only if the search rank is over-parameterized as $r>K$. Second, some special care is needed to rigorously apply the guarantees from \Cref{subsec:Riemann}, given that the penalty $\varphi(V,Q)$ is Lipschitz only when restricted to a closed and strictly feasible subset; see \Cref{appsec:Lipschitz} for details.

Together, these ingredients allow us to apply Riemannian gradient descent \citep{boumal2019global} to (\ref{eq:manif2}) to compute an $\epsilon$-first-order critical point in $(n/\epsilon)\cdot\poly(r,d,K)$
time. In practice, the algorithm often converges to an $\epsilon$-second-order critical point, though this is not rigorously guaranteed without a carefully-tuned noise perturbation. Alternatively, we can apply the
conjugate-gradients (CG) variant of the Riemannian trust-region algorithm (RTR), a general-purpose solver available in packages like \texttt{MANOPT}~\citep{manopt} or \texttt{PYMANOPT}~\citep{Pymanopt}, to guarantee convergence to an $\epsilon$-second-order critical point. Unfortunately, in our experiments, we observed that all of these algorithms experience unsatisfactorily slow convergence, due to the severe ill-conditioning introduced by the logarithmic penalty.

Instead, our best numerical results were obtained by the Riemannian cubic-regularized Newton (\Cref{thm:rcr}). 
Our key insight is that the algorithm can be implemented with just $O(nr^3)$ time per-iteration, by exploiting the underlying block-diagonal-plus-low-rank structure of the Riemannian Hessian. To explain, our core difficulty is to efficiently solve the Newton subproblem
\begin{equation}\label{eqn:newton_cube}
\min_{Ap=0}g^\top p+\frac{1}{2}p^\top Hp+\frac{L}{6}\norm{p}^{3},
\end{equation}
where $g$ and $H$ denote the vectorized Riemannian gradient and Hessian respectively, and $A$ implements the tangent space constraint $(\dot{V},\dot{Q})\in\rT_{(V,Q)}\cM$. We can verify that the subproblem contains $n(r-1)+r^{2}=O(nr)$ variables and is subject to $m=r+r(r+1)/2=O(r^{2})$ constraints.
Given that the subproblem has only linear constraints, its local minima must always satisfy the first- and second-optimality conditions (\ref{eq:FOC}) and (\ref{eq:SOC}), which read
\begin{align*}
\begin{bmatrix}H+\lambda I & A^\top \\
A & 0
\end{bmatrix}\begin{bmatrix}p\\
q
\end{bmatrix}&=\begin{bmatrix}-g\\
0
\end{bmatrix},\quad
\lambda=\frac{L}{2}\norm{p},\quad
\xi^\top (H+\lambda I)\xi\ge0\text{ for all }\xi\text{ satisfying }A\xi=0.
\end{align*}
The following result can be viewed as a Riemannian extension of known results on the approximate minimization of cubic-regularized subproblems. It shows that, with sufficient regularization $L$, the global minimum corresponds to the unique second-order critical point.
\begin{lem}\label{lem:p_lambda}
Let $A$ have full row-rank (i.e.\ $AA^\top \succ0$) and let $\lambda_{\min}=\min_{\norm{\xi}=1,A\xi=0}\xi^\top H\xi$. For $\lambda>-\lambda_{\min}$, the parameterized solution
\[
p(\lambda)=\begin{bmatrix}I\\
0
\end{bmatrix}^\top \begin{bmatrix}H+\lambda I & A^\top \\
A & 0
\end{bmatrix}^{-1}\begin{bmatrix}-g\\
0
\end{bmatrix}
\]
is well-defined and $\norm{p(\lambda)}$ is monotonously decreasing with respect to $\lambda$.
\end{lem}
The same lemma also suggests solving the Newton subproblem by simple
bisection search, cf.~\citep[Sec.~6.1]{Cartis2011}. Indeed, the solution is just $p(\lambda_{\text{opt}})$,
where $\lambda_{\text{opt}}$ is the solution to the \emph{monotone}
equation $2\lambda=L\norm{p(\lambda)}$ (via \Cref{lem:p_lambda}).
Thus, we pick a very small $\lambda_{\text{lb}}\approx-\lambda_{\min}$
such that $\norm{p(\lambda_{\text{lb}})}>2\lambda_{\text{lb}}/L$,
a very large $\lambda_{\text{ub}}$ such that $2\lambda_{\text{ub}}/L>\norm{p(\lambda_{\text{ub}})}$,
and then perform bisection until $2\lambda_{\text{opt}}=L\norm{p(\lambda_{\text{opt}})}$
is approximately found. For each $\lambda$, if $2\lambda<L\norm{p(\lambda)}$,
then we increase $\lambda$; otherwise, we decrease $\lambda$.
 
The main cost of the bisection search is the computation of $p(\lambda)$, which naively costs $O(n^3r^3)$ time. For our specific problem, we explain in \Cref{appsec:proof_bisect} how a block-diagonal-plus-low-rank structure in the Hessian $H$ reduces the computation cost to just $n\cdot\poly(r,d)$ time. 
Applying \Cref{thm:rcr} shows that the overall method computes an $\epsilon$-second-order
critical point in $(n/\epsilon^{1.5})\cdot\poly(r,d,K)$ time.

\section{Numerical results}
\label{sec:numerics}
In this section, we showcase the superior performance of our proposed Riemannian second-order method for clustering on both synthetic Gaussian mixture models (GMM) and real-world mass cytometry (CyTOF) datasets. Compared to existing state-of-the-art methods, such as the nonnegative low-rank (NLR) factorization \citep{ZhuangChenYangZhang2024_ICLR} and prior Riemannian $K$-means algorithms \citep{CarsonMixonVillarWard_manifold-Kmeans}, our approach achieves faster convergence, higher clustering accuracy, and more reliable recovery of ground-truth cluster memberships. 
These results highlight the convergence and accuracy advantages of second-order methods when they can be implemented with per-iteration costs of just $O(n)$ time.  
The implementation details are deferred to \Cref{appsec:RTR_algo}. 

\paragraph{Datasets.} We conducted experiments on both synthetic and real datasets. The synthetic data was generated from a standard $K$-component, $d$-dimensional Gaussian mixture model (GMM), with centroids placed at simplex vertices such that their separation equals $\gamma\overline{\Theta}^2$, where $\overline{\Theta}$ is the information-theoretic threshold for exact recovery in~(\ref{eqn:recovery_threshold}), and $\gamma$ controls separation. The real dataset came from mass cytometry (CyTOF) \citep{LEVINE2015184, CyTOFClean}. It consists of 265,627 cell protein expression profiles across 32 markers, labeled into 14 gated cell populations. Following~\citet{ZhuangChenYangZhang2024_ICLR}, we uniformly sample 1,800 cells from $K = 4$ unbalanced clusters (labels 2, 7, 8, 9) from individual 1 for our experiment.

\paragraph{Global optimality at second-order critical points (validation of Assumption~\ref{asm:bengin}).} Figure~\ref{fig:2opt} shows the convergence behaviors of loss function~(\ref{eq:manif2}) for GMMs ($n=500, \gamma=1.2,\mu=0.01$) with $50$ randomized initializations. We consistently observe that: (i) the loss value steadily decreases over iterations and converges rapidly near the globally optimal point; (ii) the Riemannian gradient norm dynamics suggest that our algorithm initially attempts to escape saddle points (with increased gradient norm) and eventually converges to second-order local optimality, where zero-loss is achieved, indicating global optimality. To verify second-order local optimality, we also plot the minimum eigenvalue of the Riemannian Hessian. This provides strong numerical evidence that near-second-order critical points are near-globally optimal, as posited by Assumption~\ref{asm:bengin}.

\begin{figure}[!htbp]
    \centering
    \includegraphics[width=\textwidth]{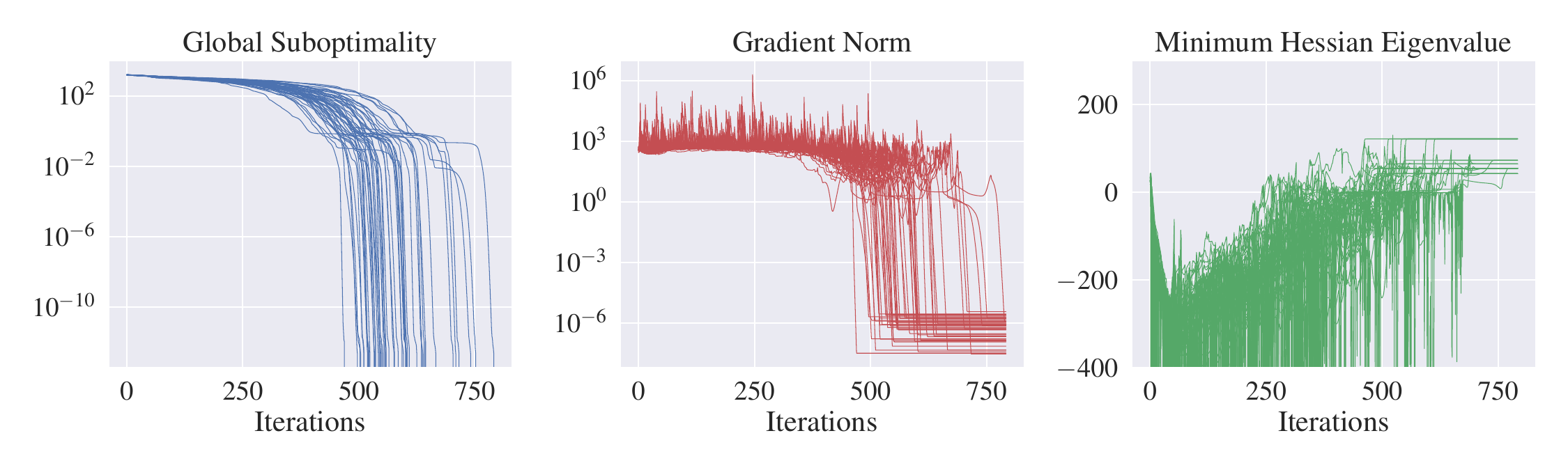}
    \caption{\textbf{Local convergence to second-order critical points yields global optimality.} In the GMM setting, where ground-truth partitions can be planted, we consistently observe local convergence to the global optimum, yielding zero clustering error. This provides strong numerical evidence that near-second-order critical points are near-globally optimal, as hypothesized in Assumption~\ref{asm:bengin}.} \label{fig:2opt}
\end{figure}

\paragraph{Benchmark on real world data.} 
Prior studies on mass cytometry (CyTOF) and computer vision (CIFAR‑10) datasets identified the nonnegative low-rank (NLR) factorization \citep{ZhuangChenYangZhang2024_ICLR} as the most reliable clustering solver, attaining the lowest average mis-clustering error and the tightest variance compared to classical baselines such spectral clustering (SC), nonnegative matrix factorization (NMF), and $K$-means++~\citep{kmeans++} ($K$M++). Our algorithm optimizes the same nonnegative low‑rank model, so it inherits this reliability. Because it applies second‑order Hessian updates rather than first‑order gradients, it refines each iterate more thoroughly and therefore recovers the ground‑truth membership matrix more accurately. \Cref{fig:cytof_perf} illustrates this on CyTOF: both methods keep mis-clustering near zero, yet our solver achieves a smaller Frobenius gap to the oracle solution. The experiment was repeated 50 times on random subsamples of size $n=1800$. An additional experiment on the CIFAR10 dataset can be found in \Cref{appsec:add_num}.

\begin{figure}[!htbp]
    \centering
    \includegraphics[width=0.8\textwidth]{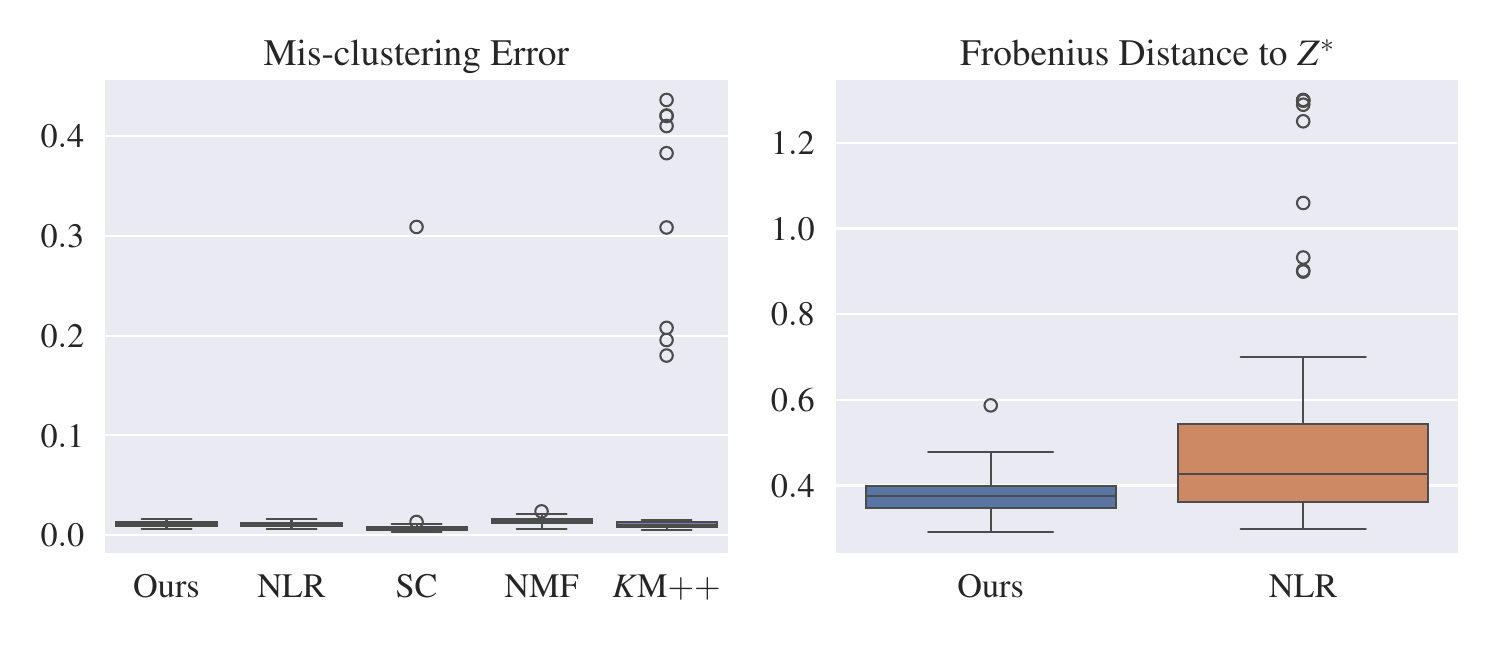}
    \caption{\textbf{Real-world benchmark on CyTOF data.} We compared our method to NLR, the previous state-of-the-art, as well as classical benchmarks SC, NMF, and $K$M++. Our method and NLR achieve the most consistently accurate clustering, with the smallest variance and the fewest outliers (left), but we outperform NLR in ground truth recovery (right).}\label{fig:cytof_perf}
\end{figure}

\paragraph{Comparison with NLR.} Next, we compare our method directly to the nonnegative low-rank (NLR) factorization. \Cref{fig:nlr_vs_tr} shows experimental results for GMM with \(n=100\), \(\gamma=0.8\), \(\mu=0.1\), and with varying $n$.  Main observations: (i) Each Newton step is solved in $O(n)$ time, matching the theory in \Cref{sec:formulation}. (ii) A Newton step is about $25$--$100$ times costlier than a single NLR update. This is to be expected because the Newton step solves several linear systems, while NLR performs only a single matrix-vector product. (iii) Our solver reaches the optimum in hundreds of iterations, whereas NLR needs tens of thousands. The orders‑of‑magnitude reduction in iterations more than offsets the costlier step, so wall‑clock time drops by a factor of two to four.

\begin{figure}[!htbp]
    \centering
    \includegraphics[width=\textwidth]{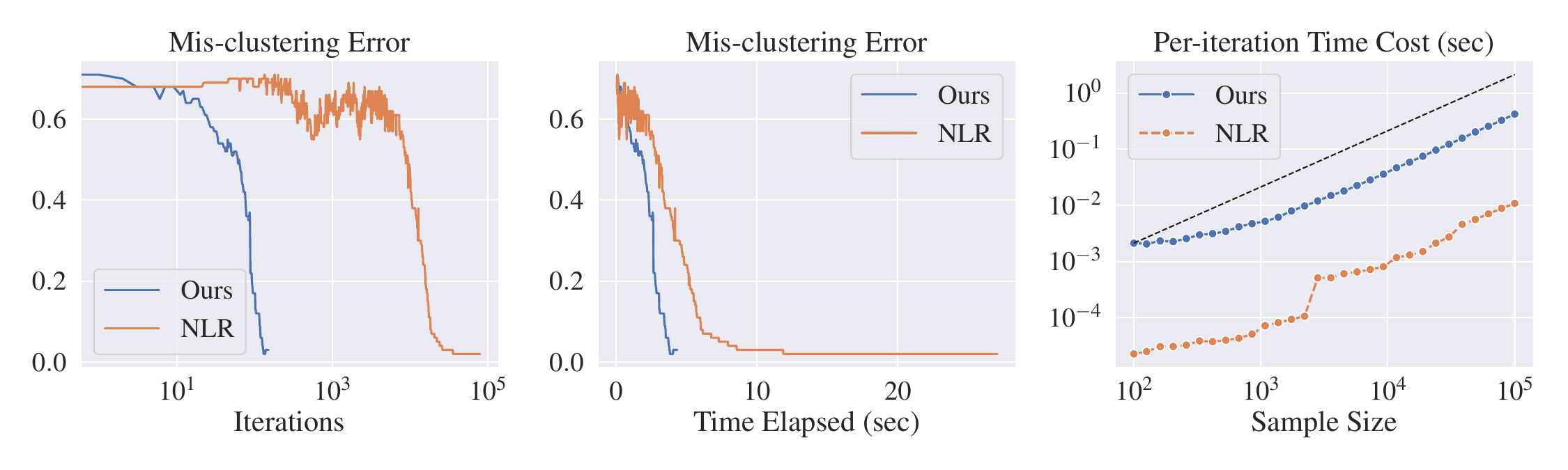}
    \caption{\textbf{Comparison with previous state-of-the-art NLR on GMM.}  
    Our second-order method reaches optimality in 152 iterations, while NLR needs 80k. Even though each second-order iteration costs $\approx$ 25--100 NLR steps, the total runtime is still two to four times shorter.
    (Left and middle) clustering accuracy vs log iterations and linear time. (Right) per-iteration time vs sample size $n$.}\label{fig:nlr_vs_tr}
\end{figure}

\begin{figure}[!htbp]
    \centering
    \includegraphics[width=\textwidth]{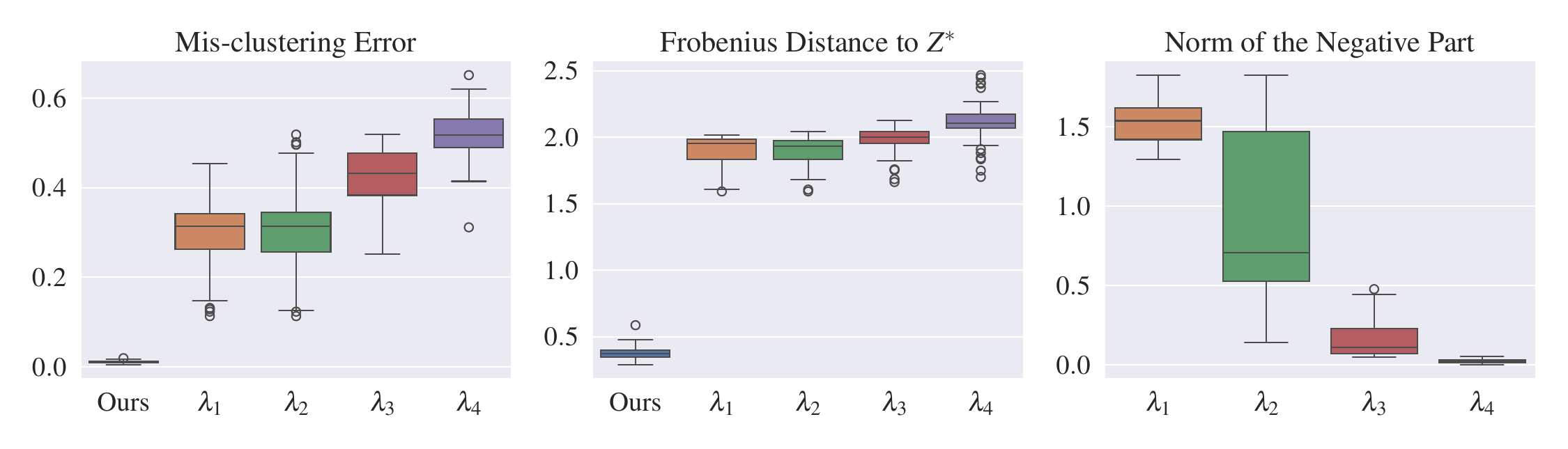}
    \caption{\textbf{Comparison with prior Riemannian $K$-means method of \citet{CarsonMixonVillarWard_manifold-Kmeans} on real-world data.} Each run is warm‑started from the previous and the penalty is stepped through $\lambda_i=0,10^4, 10^6,10^7$. However: (Left) average mis‑clustering exceeds 30\%; (Middle) the recovery error $\norm{Z-Z^\star}_F$ remains large; (Right) the infeasibility $\norm{U_-}$ never vanishes. Our Riemannian method, shown for reference, enforces $U_-=0$ by design and achieves near-zero error in both metrics.}\label{fig:cytof_fm}
\end{figure}

\paragraph{Comparison with prior Riemannian $K$-means methods.} We evaluate the clustering performance of the algorithm proposed by~\citet{CarsonMixonVillarWard_manifold-Kmeans} to solve the penalized formulation~(\ref{eqn:Carson_manifold}) on the CyTOF data. Figure~\ref{fig:cytof_fm} presents the performance of this first-order manifold method. Unfortunately, we were unable to identify a sequence of $\lambda$ that would produce acceptable clustering results. While increasing the penalty parameter $\lambda$ improves feasibility, it also degrades clustering performance. This highlights the difficulty in solving~(\ref{eqn:Carson_manifold}) using the method of~\citet{CarsonMixonVillarWard_manifold-Kmeans}, as it struggles to balance strict constraint satisfaction with objective optimality in the $K$-means problem.

\begin{figure}[!hb]
    \centering
    \includegraphics[width=0.8\textwidth]{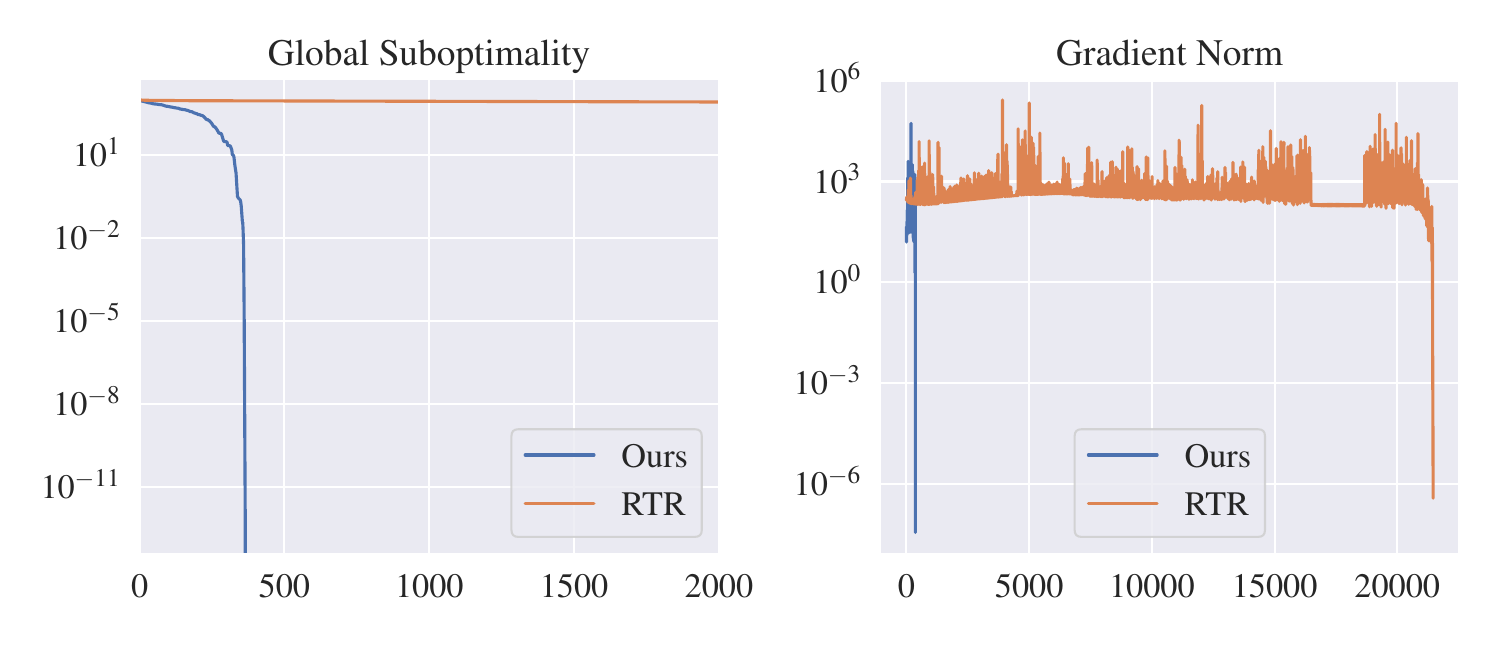}
    \caption{\textbf{Comparison with classical Riemannian Trust Region (RTR) on GMM.} Our method drives both loss and gradient norm to machine precision in around 360 iterations. In contrast, RTR stagnates for over 21k iterations due to the extreme ill-conditioning induced by the log penalty.}\label{fig:vs_rtr}
\end{figure}

\paragraph{Comparison with classical Riemannian algorithms.}
As discussed in \Cref{sec:formulation}, Problem (\ref{eq:manif2}) can be solved with CG or RTR as the gradient, Hessian, and retraction are all available. Nevertheless, both CG and RTR perform poorly because the log-barrier induces an extremely ill-conditioned landscape. To illustrate this, we solve (\ref{eq:manif2}) on GMM data using \texttt{PYMANOPT}'s implementation of CG and RTR. For CG, we were unable to tune the method to produce a meaningful solution, as its updates frequently lead to infeasible points. While RTR can converge to a solution comparable to ours, it requires significantly more iterations and time, as shown in \Cref{fig:vs_rtr}.
\begin{figure}[!htbp]
    \centering
    \includegraphics[width=0.8\textwidth]{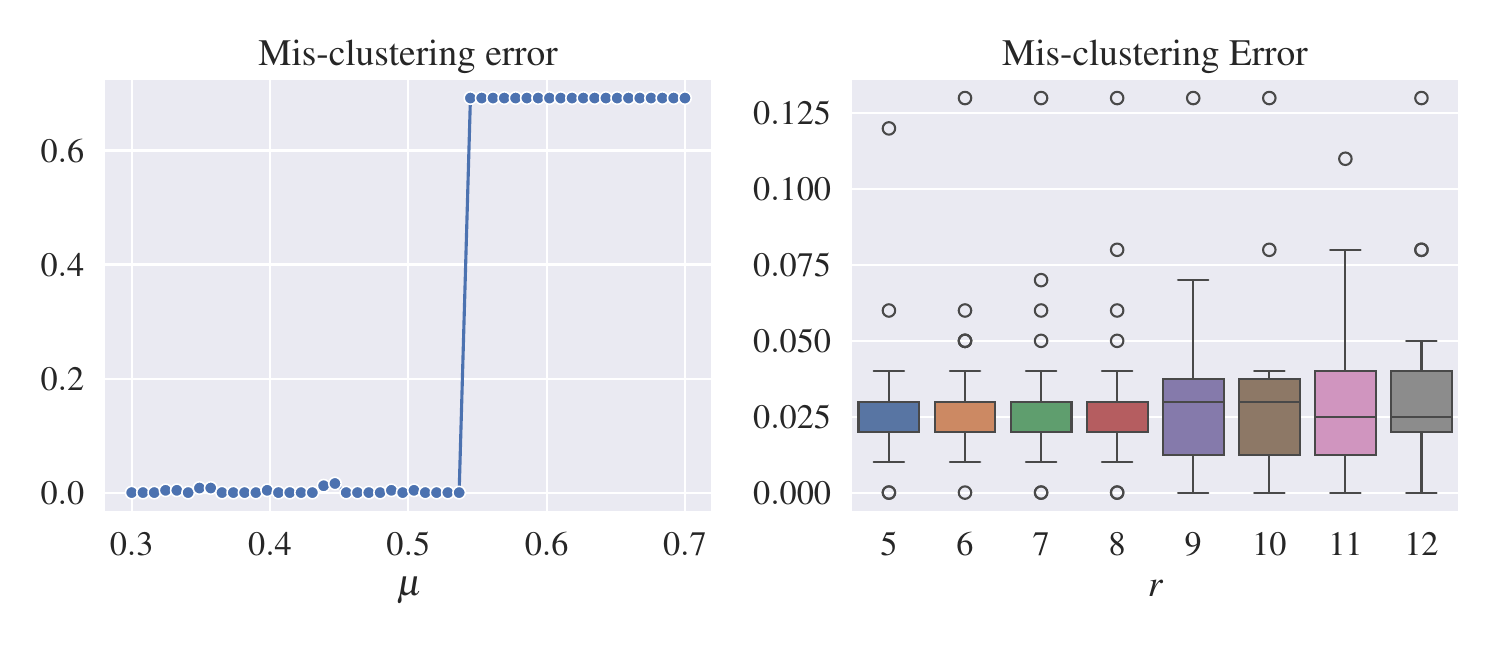}
    \caption{\textbf{Dependence on hyperparameters.} (Left) Error exhibits phase transition as penalty parameter $\mu$ gets too large. (Middle): Errors are insensitive to the search rank $r$, when $\mu$ is chosen appropriately. (Right): When $\mu$ is too large and the algorithm becomes trapped in local minima, increasing $r$ can leads to significantly better results.}\label{fig:err_rate_hyper}
\end{figure}

\begin{figure}[!htbp]
    \centering
    \includegraphics[width=\textwidth]{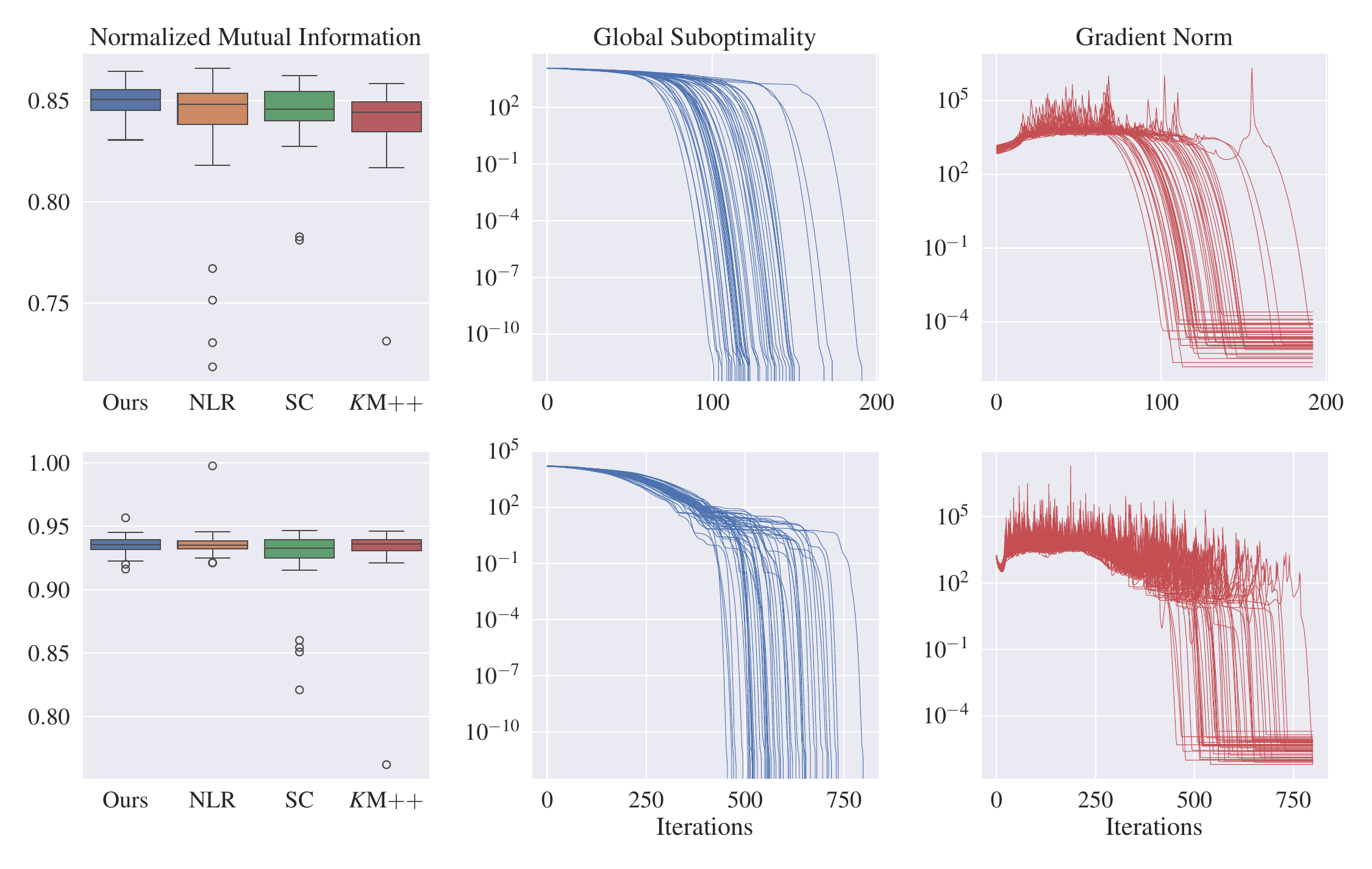}
    \caption{\textbf{Robustness to mis-specification.} Clustering performance and convergence behavior when the number of clusters is under-estimated (top) and over-estimated (bottom) \label{fig:mis_K}}
\end{figure}

\paragraph{Effect of hyperparameters.} Our method displays a sharp phase transition with respect to the regularization parameter $\mu$. Below a critical threshold, the algorithm consistently converges to optimal solutions and remains robust to variations in $\mu$. However, once $\mu$ exceeds this threshold, the method fails to yield meaningful results. This behavior is illustrated in~\Cref{fig:err_rate_hyper} where we used GMM synthetic data with four clusters and separation $\gamma=0.8$. \Cref{fig:err_rate_hyper} also shows the clustering errors evaluated across different search ranks $r$. Clustering performance is largely insensitive to $r$ when $\mu$ is small, so the smallest feasible $r=K+1$ is recommended in practice. However, when $\mu$ is large enough to trigger the the phase transition, increasing $r$ can help the algorithm escape spurious regions. A brief discussion on the choice of hyperparameters is provided in \Cref{appsec:add_num}.

\paragraph{Robustness to mis-specified cluster number.}
\Cref{fig:mis_K} shows the performance of our method under mis-specified cluster numbers in the GMM setting with $K=4, n=1000$, and $\gamma=0.8$; The mis-specified cluster number is set to be $K_\text{mis}=3$ and $K_\text{mis}=5$. In both setups, we observed that our method exhibits strong robustness in statistical accuracy for label recovery. We used normalized mutual information (NMI) to quantify the agreement between true and recovered labels, as the two sets of labels have different categories. Moreover, we observed that our method locally converges to the global optima in a similar manner as the $K=4$ case (cf.~\Cref{fig:2opt}).

\section*{Acknowledgments}
X.~Chen was partially supported by NSF DMS-2413404 and a gift from the Simons Foundation. R.~Zhang was partially supported by NSF CAREER Award ECCS-2047462 and ONR Award N00014-24-1-2671.

\section*{Reproducibility Statement}
The experimental setups are described in detail in \Cref{sec:numerics} and \Cref{appsec:add_num}. While the provided code does not encompass every experiment reported in the paper, all of them can be readily reproduced based on the descriptions.

\section*{LLM Usage Disclosure}
A large language model (LLM) was used solely to polish the writing and improve clarity of expression. No part of the research ideation, discovery, or substantive content was generated by the LLM.

\newpage
\printbibliography

@misc{zhang2025nonnegative,
    title={Nonnegative Low-rank Matrix Recovery Can Have Spurious Local Minima},
    author={Zhang, Richard Y},
    year={2025},
    archivePrefix = {arXiv},
    eprint = {2505.03717},
    primaryClass = {math.OC}
}

@article{dickinson2014computational,
    title={On the computational complexity of membership problems for the completely positive cone and its dual},
    author={Dickinson, Peter JC and Gijben, Luuk},
    journal={Computational optimization and applications},
    volume={57},
    pages={403--415},
    year={2014},
    publisher={Springer}
}

@article{bhojanapalli2016global,
    title={Global optimality of local search for low rank matrix recovery},
    author={Bhojanapalli, Srinadh and Neyshabur, Behnam and Srebro, Nati},
    journal={Advances in Neural Information Processing Systems},
    volume={29},
    year={2016}
}

@article{ge2016matrix,
    title={Matrix completion has no spurious local minimum},
    author={Ge, Rong and Lee, Jason D and Ma, Tengyu},
    journal={Advances in neural information processing systems},
    volume={29},
    year={2016}
}

@inproceedings{bandeira2016low,
    title={On the low-rank approach for semidefinite programs arising in synchronization and community detection},
    author={Bandeira, Afonso S and Boumal, Nicolas and Voroninski, Vladislav},
    booktitle={Conference on learning theory},
    pages={361--382},
    year={2016},
    organization={PMLR}
}

@article{byrd2000trust,
    title={A trust region method based on interior point techniques for nonlinear programming},
    author={Byrd, Richard H and Gilbert, Jean Charles and Nocedal, Jorge},
    journal={Mathematical programming},
    volume={89},
    pages={149--185},
    year={2000},
    publisher={Springer}
}

@article{byrd2006k,
    title={Knitro: An integrated package for nonlinear optimization},
    author={Byrd, Richard H and Nocedal, Jorge and Waltz, Richard A},
    journal={Large-scale nonlinear optimization},
    pages={35--59},
    year={2006},
    publisher={Springer}
}

@inproceedings{kmeans++,
    author = {Arthur, David and Vassilvitskii, Sergei},
    title = {\texttt{k-means++}: the advantages of careful seeding},
    year = {2007},
    isbn = {9780898716245},
    publisher = {Society for Industrial and Applied Mathematics},
    address = {USA},
    booktitle = {Proceedings of the Eighteenth Annual ACM-SIAM Symposium on Discrete Algorithms},
    pages = {1027--1035},
    numpages = {9},
    location = {New Orleans, Louisiana},
    series = {SODA '07}
}

@article{boumal2019global,
    title={Global rates of convergence for nonconvex optimization on manifolds},
    author={Boumal, Nicolas and Absil, Pierre-Antoine and Cartis, Coralia},
    journal={IMA Journal of Numerical Analysis},
    volume={39},
    number={1},
    pages={1--33},
    year={2019},
    publisher={Oxford University Press}
}

@misc{zhang2018cubic,
    title={A cubic regularized Newton's method over Riemannian manifolds},
    author={Zhang, Junyu and Zhang, Shuzhong},
    year={2018},
    archivePrefix = {arXiv},
    eprint = {1805.05565},
    primaryClass = {math.OC}
}

@article{agarwal2021adaptive,
    title={Adaptive regularization with cubics on manifolds},
    author={Agarwal, Naman and Boumal, Nicolas and Bullins, Brian and Cartis, Coralia},
    journal={Mathematical Programming},
    volume={188},
    pages={85--134},
    year={2021},
    publisher={Springer}
}

@article{murty1987some,
    title={Some {NP}-complete problems in quadratic and nonlinear programming},
    author={Murty, Katta G and Kabadi, Santosh N},
    journal={Mathematical Programming},
    volume={39},
    pages={117--129},
    year={1987},
    publisher={Springer}
}

@article{KALOFOLIAS2012421,
    title = {Computing symmetric nonnegative rank factorizations},
    journal = {Linear Algebra and its Applications},
    volume = {436},
    number = {2},
    pages = {421--435},
    year = {2012},
    note = {Special Issue devoted to the Applied Linear Algebra Conference (Novi Sad 2010)},
    issn = {0024-3795},
    doi = {10.1016/j.laa.2011.03.016},
    author = {V. Kalofolias and E. Gallopoulos},
    url = {https://www.sciencedirect.com/science/article/pii/S0024379511002199},
}

@book{boumal2023intromanifolds,
    title     = {{An Introduction to Optimization on Smooth Manifolds}},
    author    = {Boumal, Nicolas},
    publisher = {Cambridge University Press},
    year      = {2023},
    url       = {https://www.nicolasboumal.net/book},
    doi       = {10.1017/9781009166164}
}

@article{Pymanopt,
    author = {James Townsend and Niklas Koep and Sebastian Weichwald},
    journal = {Journal of Machine Learning Research},
    number = {137},
    pages = {1--5},
    title = {{Pymanopt: A Python Toolbox for Optimization on Manifolds using Automatic Differentiation}},
    url = {http://jmlr.org/papers/v17/16-177.html},
    volume = {17},
    year = {2016}
}

@article{manopt,
    author  = {Boumal, N. and Mishra, B. and Absil, P.-A. and Sepulchre, R.},
    journal = {Journal of Machine Learning Research},
    title   = {{M}anopt, a {M}atlab Toolbox for Optimization on Manifolds},
    year    = {2014},
    number  = {42},
    pages   = {1455--1459},
    volume  = {15},
    url     = {https://www.manopt.org}
}

@article{ChenYang_nonEuclideanKmeans,
    author = {Chen, Xiaohui and Yang, Yun},
    doi = {10.3150/20-BEJ1251},
    journal = {Bernoulli},
    number = {1},
    pages = {586 -- 614},
    publisher = {Bernoulli Society for Mathematical Statistics and Probability},
    title = {{Hanson--Wright inequality in Hilbert spaces with application to $K$-means clustering for non-Euclidean data}},
    volume = {27},
    year = {2021},
}

@INPROCEEDINGS{CarsonMixonVillarWard_manifold-Kmeans,
    author={Carson, Timothy and Mixon, Dustin G. and Villar, Soledad and Ward, Rachel},
    booktitle={2017 International Conference on Sampling Theory and Applications (SampTA)}, 
    title={Manifold optimization for $k$-means clustering}, 
    year={2017},
    volume={},
    number={},
    pages={73--77},
    doi={10.1109/SAMPTA.2017.8024388}
}

@inproceedings{ZhuangChenYangZhang2024_ICLR,
    author = {Zhuang, Yubo and Chen, Xiaohui and Yang, Yun and Zhang, Richard Y.},
    booktitle={The Twelfth International Conference on Learning Representations},
    title = {{Statistically Optimal $K$-means Clustering via Nonnegative Low-rank Semidefinite Programming}},
    year = {2024}
}

@inproceedings{ge2017no,
    title={No spurious local minima in nonconvex low rank problems: A unified geometric analysis},
    author={Ge, Rong and Jin, Chi and Zheng, Yi},
    booktitle={International Conference on Machine Learning},
    pages={1233--1242},
    year={2017},
    organization={PMLR}
}

@article{boumal2020deterministic,
  title={{Deterministic Guarantees for Burer-Monteiro Factorizations of Smooth Semidefinite Programs}},
  author={Boumal, Nicolas and Voroninski, Vladislav and Bandeira, Afonso S},
  journal={Communications on Pure and Applied Mathematics},
  volume={73},
  number={3},
  pages={581--608},
  year={2020},
  publisher={Wiley Online Library}
}

@article{boumal2016non,
  title={The non-convex Burer-Monteiro approach works on smooth semidefinite programs},
  author={Boumal, Nicolas and Voroninski, Vlad and Bandeira, Afonso},
  journal={Advances in Neural Information Processing Systems},
  volume={29},
  year={2016}
}

@ARTICLE{6061964,
    author={He, Zhaoshui and Xie, Shengli and Zdunek, Rafal and Zhou, Guoxu and Cichocki, Andrzej},
    journal={IEEE Transactions on Neural Networks}, 
    title={Symmetric Nonnegative Matrix Factorization: Algorithms and Applications to Probabilistic Clustering}, 
    year={2011},
    volume={22},
    number={12},
    pages={2117--2131},
    doi={10.1109/TNN.2011.2172457}
}

@inproceedings{NgJordanWeiss2001_NIPS,
	author = {Andrew Y. Ng and Michael I. Jordan and Yair Weiss},
	booktitle = {Advances in Neural Information Processing Systems},
	pages = {849--856},
	publisher = {MIT Press},
	title = {{On Spectral Clustering: Analysis and an algorithm}},
	year = {2001}}

@article{vanLuxburg2007_spectralclustering,
	author = {von Luxburg, Ulrike},
	journal = {Statistics and Computing},
	number = {4},
	pages = {395--416},
	title = {A tutorial on spectral clustering},
	volume = {17},
	year = {2007}}

@article{kuang2015symnmf,
  title={{SymNMF: nonnegative low-rank approximation of a similarity matrix for graph clustering}},
  author={Kuang, Da and Yun, Sangwoon and Park, Haesun},
  journal={Journal of Global Optimization},
  volume={62},
  pages={545--574},
  year={2015},
  publisher={Springer}
}

@article{wang2012nonnegative,
  title={Nonnegative matrix factorization: A comprehensive review},
  author={Wang, Yu-Xiong and Zhang, Yu-Jin},
  journal={IEEE Transactions on knowledge and data engineering},
  volume={25},
  number={6},
  pages={1336--1353},
  year={2012},
  publisher={IEEE}
}

@article{aloise2009np,
	author = {Aloise, Daniel and Deshpande, Amit and Hansen, Pierre and Popat, Preyas},
	journal = {Machine learning},
	number = {2},
	pages = {245--248},
	title = {{NP-hardness of Euclidean sum-of-squares clustering}},
	volume = {75},
	year = {2009}}

@article{Dasgupta2007,
	author = {Dasgupta, Sanjoy},
	journal = {Technical Report CS2007-0890, University of California, San Diego},
	title = {The hardness of $k$-means clustering},
	year = {2007}}

@article{Lloyd1982_TIT,
	author = {Lloyd, Stuart},
	journal = {IEEE Transactions on Information Theory},
	pages = {129--137},
	title = {Least squares quantization in {PCM}},
	volume = {28},
	year = {1982}}

@article{PengWei2007_SIAMJOPTIM,
	author = {Peng, Jiming and Wei, Yu},
	journal = {SIAM J. OPTIM},
	number = {1},
	pages = {186--205},
	title = {Approximating ${K}$-means-type clustering via semidefinite programming},
	volume = {18},
	year = {2007}}

@InProceedings{FeiChen2018,
  title = 	 {Hidden Integrality of SDP Relaxations for Sub-Gaussian Mixture Models},
  author =       {Fei, Yingjie and Chen, Yudong},
  booktitle = 	 {Proceedings of the 31st  Conference On Learning Theory},
  pages = 	 {1931--1965},
  date = {2018-07-06},
  editor = 	 {Bubeck, Sébastien and Perchet, Vianney and Rigollet, Philippe},
  volume = 	 {75},
  series = 	 {Proceedings of Machine Learning Research},
  publisher =    {PMLR},
  url = 	 {https://proceedings.mlr.press/v75/fei18a.html}
}

@article{chen2021cutoff,
	author = {Chen, Xiaohui and Yang, Yun},
	journal = {IEEE Transactions on Information Theory},
	number = {6},
	pages = {4223--4238},
	title = {{Cutoff for exact recovery of Gaussian mixture models}},
	volume = {67},
	year = {2021}
}

@inproceedings{Royer2017_NIPS,
	author = {Royer, Martin},
	booktitle = {Advances in Neural Information Processing Systems 30},
	editor = {I. Guyon and U. V. Luxburg and S. Bengio and H. Wallach and R. Fergus and S. Vishwanathan and R. Garnett},
	pages = {1795--1803},
	publisher = {Curran Associates, Inc.},
	title = {Adaptive Clustering through Semidefinite Programming},
	year = {2017}
}

@article{BurerMonteiro2003,
    author = {Burer, Samuel and Monteiro, Renato D. C.},
    date = {2003-02-01},
    doi = {10.1007/s10107-002-0352-8},
    id = {Burer2003},
    issn = {1436-4646},
    journal = {Mathematical Programming},
    number = {2},
    pages = {329--357},
    title = {A nonlinear programming algorithm for solving semidefinite programs via low-rank factorization},
    volume = {95},
    year = {2003},
}

@inproceedings{ZhuangChenYang2022_NeurIPS,
	author = {Zhuang, Yubo and Chen, Xiaohui and Yang, Yun},
	booktitle = {Advances in Neural Information Processing Systems},
	title = {Wasserstein {$K$}-means for clustering probability distributions},
	year = {2022}
}

@book{NocedalWright2006_NumOpt,
    author = {Jorge Nocedal and Stephen J. Wright},
    title = {Numerical Optimization},
    series = {Springer Series in Operations Research and Financial Engineering},
    publisher = {Springer New York, NY},
    year = {2006},
    doi = {10.1007/978-0-387-40065-5},
    isbn = {978-0-387-40065-5},
    edition = {2nd}
}

@book{klerk2006,
    title={Aspects of Semidefinite Programming},
    author={Etienne Klerk},
    year={2006},
    publisher={Springer New York, NY},
    isbn={978-0-306-47819-2},
    doi={10.1007/b105286}
}

@Article{Levin2025,
	author={Levin, Eitan
	and Kileel, Joe
	and Boumal, Nicolas},
	title={The effect of smooth parametrizations on nonconvex optimization landscapes},
	journal={Mathematical Programming},
	year={2025},
	month={01},
	day={01},
	volume={209},
	number={1},
	pages={63--111},
	issn={1436-4646},
	doi={10.1007/s10107-024-02058-3}
}

@article{LEVINE2015184,
    title = {Data-Driven Phenotypic Dissection of AML Reveals Progenitor-like Cells that Correlate with Prognosis},
    journal = {Cell},
    volume = {162},
    number = {1},
    pages = {184--197},
    year = {2015},
    issn = {0092-8674},
    doi = {10.1016/j.cell.2015.05.047},
    url = {https://www.sciencedirect.com/science/article/pii/S0092867415006376},
    author = {Jacob H. Levine and Erin F. Simonds and Sean C. Bendall and Kara L. Davis and El-ad D. Amir and Michelle D. Tadmor and Oren Litvin and Harris G. Fienberg and Astraea Jager and Eli R. Zunder and Rachel Finck and Amanda L. Gedman and Ina Radtke and James R. Downing and Dana Pe’er and Garry P. Nolan},
}

@misc{CyTOFClean,
    title = {Clustering benchmark data: 32-dimensional data set from {Levine et al.}},
    author = {Lukas Weber},
    year = {2015},
    note = {GitHub Repository},
    url={https://github.com/lmweber/benchmark-data-Levine-32-dim}
}

@article{mishra2016riemannian,
  title={Riemannian preconditioning},
  author={Mishra, Bamdev and Sepulchre, Rodolphe},
  journal={SIAM Journal on Optimization},
  volume={26},
  number={1},
  pages={635--660},
  year={2016},
  publisher={SIAM}
}

@misc{absil2009all,
    title={All roads lead to Newton: Feasible second-order methods for equality-constrained optimization},
    author={Absil, P-A and Trumpf, Jochen and Mahony, Robert and Andrews, Ben},
    year={2009},
    note = {Technical Report},
}

@ARTICLE{9585110,
    author={Qian, Wei and Zhang, Yuqian and Chen, Yudong},    
    journal={IEEE Transactions on Information Theory},     
    title={Structures of Spurious Local Minima in k-Means},     
    year={2022},    
    volume={68},    
    number={1},    
    pages={395-422},    
    doi={10.1109/TIT.2021.3122465}
}

@book{magnus99,
    author = {Magnus, Jan R. and Neudecker, Heinz},
    edition = {Third},
    series = {Wiley Series in Probability and Statistics},
    isbn = {9781119541202},
    publisher = {John Wiley},
    doi = {10.1002/9781119541219},
    title = {Matrix Differential Calculus with Applications in Statistics and Econometrics},
    year = {2019}
}

@Article{Huang2025,
    author={Huang, Wen
    and Wei, Meng
    and Gallivan, Kyle A.
    and Van Dooren, Paul},
    title={A Riemannian Optimization Approach to Clustering Problems},
    journal={Journal of Scientific Computing},
    year={2025},
    month={2},
    day={12},
    volume={103},
    number={1},
    pages={8},
    issn={1573-7691},
    doi={10.1007/s10915-025-02806-3},
}

@book{boyd2004convex,
	author = {Boyd, S.P. and Vandenberghe, L.},
	isbn = {9780521833783},
	lccn = {03063284},
	publisher = {Cambridge University Press},
	title = {Convex Optimization},
	year = {2004},
}

@article{Cartis2011,
    author={Cartis, Coralia
    and Gould, Nicholas I. M.
    and Toint, Philippe L.},
    title={{Adaptive cubic regularisation methods for unconstrained optimization. Part I: motivation, convergence and numerical results}},
    journal={Mathematical Programming},
    year={2011},
    month={4},
    day={01},
    volume={127},
    number={2},
    pages={245-295},
    issn={1436-4646},
    doi={10.1007/s10107-009-0286-5},
}

@inproceedings{dosovitskiy2021an,
    title={An Image is Worth 16x16 Words: Transformers for Image Recognition at Scale},
    author={Alexey Dosovitskiy and Lucas Beyer and Alexander Kolesnikov and Dirk Weissenborn and Xiaohua Zhai and Thomas Unterthiner and Mostafa Dehghani and Matthias Minderer and Georg Heigold and Sylvain Gelly and Jakob Uszkoreit and Neil Houlsby},
    booktitle={International Conference on Learning Representations},
    year={2021},
    url={https://openreview.net/forum?id=YicbFdNTTy}
}

@Article{Wang2025,
    author={Wang, Jie and Hu, Liangbing},
    title={Solving Low-Rank Semidefinite Programs via Manifold Optimization},
    journal={Journal of Scientific Computing},
    year={2025},
    month={5},
    day={31},
    volume={104},
    number={1},
    pages={33},
    issn={1573-7691},
    doi={10.1007/s10915-025-02952-8},
}

@misc{msc2025,
    title={A low-rank augmented Lagrangian method for large-scale semidefinite programming based on a hybrid convex-nonconvex approach},
    author={Renato D.C. Monteiro and Arnesh Sujanani and Diego Cifuentes},
    note = {arXiv:2401.12490},
    year={2025}
}

@article{doi:10.1287/opre.2024.1137,
    author = {Hou, Di and Tang, Tianyun and Toh, Kim-Chuan},
    title = {A Low-Rank Augmented Lagrangian Method for Doubly Nonnegative Relaxations of Mixed-Binary Quadratic Programs},
    journal = {Operations Research},
    year={2025},
    month={10},
    publisher={INFORMS},
    issn={0030-364X},
    doi = {10.1287/opre.2024.1137},
}

@article{doi:10.1137/22M1474539,
    author = {Wang, Yifei and Deng, Kangkang and Liu, Haoyang and Wen, Zaiwen},
    title = {A Decomposition Augmented Lagrangian Method for Low-Rank Semidefinite Programming},
    journal = {SIAM Journal on Optimization},
    volume = {33},
    number = {3},
    pages = {1361-1390},
    year = {2023},
    doi = {10.1137/22M1474539},
}

\newpage
\begin{appendices}
\section{Relationship between SDP Formulations of \texorpdfstring{$K$}{K}-means Clustering: Standard and Interior Point Versions (\ref{eq:sdp})}
\label{app:SDP_formulation_equiv}

The standard $K$-means clustering for data in \(\bR^d\) has two formulations in the literature: (i) \emph{centroid-based} optimization \[\min_{\beta_1,\dots,\beta_K \in \mathbb{R}^d} \sum_{i=1}^n \min_{k \in [K]}\norm{X_i - \beta_k}_2^2;\] and (ii) \emph{partition-based} optimization \[\min_{\bigsqcup_{k=1}^K G_k = [n]} \sum_{k=1}^K \sum_{i \in G_k}\norm{X_i - \bar{X}_k}_2^2,\] where \(\bar{X}_k =\abs{G_k}^{-1}\sum_{j \in G_k} X_j\)
is the empirical centroid of cluster \(G_k\). Formulations (i) and (ii) are known to be equivalent, cf.~\citet[Eqn.~(1)]{ZhuangChenYang2022_NeurIPS} or \citet[Appx.~A]{9585110}. Using the parallelogram law in \citet[Eqn.~(5)]{ZhuangChenYang2022_NeurIPS}
\[\sum_{i,j \in G_k}\norm{X_i - X_j}_2^2 = 2\abs{G_k}\sum_{i \in G_k}\norm{X_i - \bar{X}_k}_2^2,\]
we may write the partition-based objective function as \[\min_{\bigsqcup_{k=1}^K G_k = [n]} \sum_{k=1}^K \frac{1}{2 \abs{G_k}} \sum_{i, j \in G_k}\norm{X_i - X_j}_2^2.\] Next, expanding the pairwise squared Euclidean distance and dropping
\(\sum_{i=1}^n\norm{X_i}_2^2\) (no longer depending on any partition \(G_1,\dotsc,G_K\)), we arrive at (\ref{eqn:Kmeans}), in the form of maximizing the total intra-cluster similarity in terms of the Gram matrix \(\{\inner{X_i}{X_j}\}_{i,j=1}^n\). 

For general data without a likelihood derivation as in \citet{chen2021cutoff}, we can replace the \(\bR^d\)-inner product with any (positive semidefinite) kernel \(k\colon\mathcal X \times \mathcal X \to\bR\) and consider the kernelized version of $K$-means clustering that involves data \(X_1, \dotsc, X_n\) only via their Gram matrix \(\{k(X_i, X_j)\}_{i,j=1}^n\). Thus, our manifold formulation of this paper carries over to the general kernel method setting with possibly nonlinear boundary structure.

Next, we convexify the $K$-means problem (\ref{eqn:Kmeans}) into an SDP. Note that each partition \(G_1, \dotsc, G_K\) via one-hot encoding is equivalent to an assignment matrix \(H_{n \times K}\) (up to cluster relabel) where the latter is a binary matrix with exactly one non-zero entry in each row, i.e.~\(H_{ik} = 1\) if \(i \in G_k\). With this reparameterization, one can write (\ref{eqn:Kmeans}) as a mixed zero-one integer program:
\[\max_{H \in \{0, 1\}^{n \times K}} \{\inner{X X^\top}{H B H^\top}: H\vone_K = \vone_n \}.\]
Now, applying the change of variables \(Z_{n \times n} = H B H^\top\) and noting that the membership matrix \(Z\) and assignment matrix \(H\) are not one-to-one, we relax the $K$-means problem by preserving the key properties of \(Z\) as the following constraints: 
\[Z \succeq 0,\qquad\tr(Z) = K,\qquad Z\vone_n = {1}_n,\qquad Z \geq 0,\]
which no longer depend on the assignment matrix \(H\). Then, we arrive at the standard SDP relaxation for $K$-means clustering, cf.~\citet[Eqn.~(13)]{PengWei2007_SIAMJOPTIM} or \citet[Eqn.~(11)]{chen2021cutoff}:
\begin{equation}
    \label{eqn:Kmeans_SDP_standard}
    \max_{Z \in \bR^{n \times n}} \mleft\{ \inner{X X^\top}{Z} : Z \succeq 0,\, \tr(Z) = K,\, Z\vone_n =\vone_n,\, Z \geq 0 \mright\}.
\end{equation}

In practice, the elementwise nonnegativity constraint $Z \geq 0$ is almost always enforced by a logarithmic barrier. This means that we can make (\ref{eqn:Kmeans_SDP_standard}) more explicitly in the form
\begin{equation}
    \label{eqn:Kmeans_SDP_interiorpoint}
    \max_{Z \in \bR^{n \times n}}\Biggl\{ \inner{X X^\top}{Z} + \mu \sum_{i,j=1}^n \log(Z_{i,j})_+ : Z \succeq 0,\, \tr(Z) = K,\, Z\vone_n =\vone_n\Biggr\},
\end{equation}
which is precisely how any practical interior-point solver would solve the original SDP in~(\ref{eqn:Kmeans_SDP_standard}). The barrier cost is the actual objective used internally, with $\mu$ set to reflect the solver's target accuracy, cf. \citet[Chapter 11]{boyd2004convex} or \citet[Chapters 14 \& 19]{NocedalWright2006_NumOpt}. In other words, we take the standard SDP in~(\ref{eqn:Kmeans_SDP_standard}), and make explicit the logarithmic penalty that is already implicit in how such an SDP is actually solved.

\section{Tightness of Proposed Formulation}
As shown in \citet{CarsonMixonVillarWard_manifold-Kmeans}, the combinatorial $K$-means problem (\ref{eqn:Kmeans}) is exactly equivalent to \[\max_{U\in\mathbb{R}^{n\times K}}\mleft\{\inner{XX^{\top}}{UU^{\top}}:U\geq0,UU^{\top}\vone_{n}=\vone_{n},U^{\top}U=I_{K}\mright\}.\] Formulation (\ref{eq:manif}) replaces the hard constraint 
$U\ge0$ with a log barrier, and then relax $U^{\top}U=I_{K}$ into $\tr(UU^{\top})=K$. Compare to (\ref{eq:sdp}), the containment $\{UU^{T}:U\ge0\}\subset\{Z:Z\succeq0,Z\ge0\}$ is strict in general, so replacing $UU^\top$ by $Z$ yields a relaxation. \citet{chen2021cutoff} showed that the resulting relaxation from (\ref{eqn:Kmeans}) to (\ref{eq:sdp}) is exactly tight above the statistical separation threshold: its optimizer factors as $Z=UU^{\top}$ with $U\geq 0$, and the sparsity pattern of $U$ recovers the true cluster labels. Since (\ref{eq:manif}) is tighter than (\ref{eq:sdp}), it must also be tight in this regime.

\section{Information-theoretic Threshold for Exact Recovery}\label{sec:info_threshold}

The following theorem is a precise statement of the information-theoretic threshold of~\citet{chen2021cutoff}.

\begin{thm}[Average-case phase transition for exact recovery]
\label{thm:recovery_threshold} Let $\alpha>0$ and \[\Theta_{\min}\coloneqq\min_{1\leq j<k\leq K}\norm{\mu_{j}-\mu_{k}}\]
be the minimum centroid separation. Suppose that data $X_{1},\dots,X_{n}$
are generated from the Gaussian mixture model~(\ref{eqn:GMM}) with
equal cluster size $\abs{G_{1}^{*}}=\dotsm=\abs{G_{K}^{*}}=m$. Then we have
the following dichotomy. 
\begin{enumerate}
\item If $K\leq\log{n}/\log\log(n)$ and $\Theta_{\min}\geq(1+\alpha)\overline{\Theta}$,
then there exist constants $C_{1},C_{2}>0$ depending only on $\alpha$
such that, with probability at least $1-C_{1}(\log{n})^{-C_{2}}$,
the SDP (\ref{eq:sdp}) as $\mu \to 0^+$ (cf. Appendix~\ref{app:SDP_formulation_equiv} for the equivalence of two SDP formulations) has a unique solution that exactly recovers
the true partition $G_{1}^{*},\dots,G_{K}^{*}$.
\item If $K\leq\log{n}$ and $\Theta_{\min}\leq(1-\alpha)\overline{\Theta}$,
then there exists a constant $C_{3}>0$ depending only on $\alpha$
such that 
\[
\inf_{\hat{G}_{1},\dots,\hat{G}_{K}}\sup_{\Xi(n,K,\Theta_{\min})}\bP(\exists k:\hat{G}_{k}\neq G_{k}^{*})\geq1-\frac{C_{3} K}{n},
\]
where the infimum is taken over all possible estimators $(\hat{G}_{1},\dots,\hat{G}_{K})$
for $(G_{1}^{*},\dots,G_{K}^{*})$ and the parameter space is defined
as 
\[\Xi(n,K,\Theta)\coloneqq\mleft\{(G_{1},\dots,G_{K},\mu_{1},\dots,\mu_{k})\;\middle|\;
\begin{aligned}
&\norm{\mu_{j}-\mu_{k}}\geq\Theta, \; \forall j, k\in[K],j\ne k\\
&(1-\delta_{n})m\leq\abs{G_{k}}\leq(1+\delta_{n})m
\end{aligned}\mright\}\]
with $\delta_{n}=C\sqrt{K\log(n)/n}$ for some large enough constant
$C>0$. 
\end{enumerate}
\end{thm}

\section{Proofs}\label{appsec:proofs}
\begin{proof}[Proof of Lemma~\ref{lem:membership2clusterlabel}]
Note that the membership matrix $Z$ associated to a partition $G_{1},\dotsc,G_{K}$
contains a diagonal principal submatrix of rank $K$. The lemma follows
from Theorem 4 in~\citet{KALOFOLIAS2012421}. 
\end{proof}

\begin{proof}[Proof of Lemma~\ref{lem:p_lambda}]
Let $N$ denote the null space basis of $A$, such that $AN=0$ and
$N^\top N=I$. Then, we have $\lambda_{\min}=\lambda_{\min}(N^\top HN)$
and $p(\lambda)=N\hat{p}(\lambda)$ where $(N^\top HN+\lambda I)\hat{p}(\lambda)=-N^\top g$.
Then, $\lambda>-\lambda_{\min}$ implies that $N^\top HN+\lambda I\succ0$,
so $p(\lambda)=-N(N^\top HN+\lambda I)^{-1}N^\top g$ is always well-defined.
Moreover, $\|p(\lambda)\|=\|\hat{p}(\lambda)\|$ is monotonously decreasing
because all the eigenvalues of $N^\top HN+\lambda I$ are strictly positive
and increasing with $\lambda$.
\end{proof}

\begin{proof}[Proof of Lemma~\ref{lem:feasibility}]
    To prove the first statement, note that the implication 
    \begin{equation*}
        U \in \bR_{\ge 0}^{n \times K}\land UU^\top \vone_n = \vone_n\land U^\top U = I_K \implies U \in \bR_{\ge 0}^{n \times K}\land UU^\top \vone_n = \vone_n\land \norm{U}_F^2 = K
    \end{equation*}
    is straightforward. To see the converse, note that $UU^\top$ is a (doubly) stochastic matrix, $\tr(UU^\top) = K$, and $\rank(UU^\top) = K$, thus all the $K$ eigenvalues of $UU^\top$ are $1$, \textit{i.e}.~$U\in\operatorname{St}(n,K)$.

    For the second statement, it is trivial that
    \begin{align*}
        &U = \begin{bmatrix}\dfrac{1}{\sqrt{\abs{G_1}}}\vone_{G_1} & \dots & \dfrac{1}{\sqrt{\abs{G_K}}}\vone_{G_K}
        \end{bmatrix}
        \land\bigsqcup_{k=1}^K G_k = [n]\\
        &\quad\implies 
        U \in \bR_{\geq 0}^{n \times K}\land UU^\top \vone_n = \vone_n\land U^\top U = I_K.
    \end{align*}
    The converse is also true. Let us denote $G_k \coloneqq\operatorname{supp}(u_k)$. Observe that $G_i \cap G_j = \phi$ for all $i \neq j$ since $U \geq 0$ and $u_i^\top u_j = 0$ for all $i \neq j$. Then $UU^\top\vone_n = \vone_n$ implies $u_k u_k^\top \vone_{G_k} = \vone_{G_k}$ and $\bigsqcup_{k=1}^K G_k = [n]$. Since $\norm{u_k} = 1$, we have that $u_k = \vone_{G_k}/\sqrt{|G_k|}$.

    Finally, we will prove the third statement. Let $U$ be a group assignment matrix as defined by~(\ref{eqn:group_assignment_mat}). Note that for all $V \in\rT_U \cM\cap \cC_U$, $V$ must satisfy $(UV^\top + VU^\top) \vone_n = \vzero_n$, $\inner{U}{V}=0$, and $v_{i,j} \ge 0, \forall u_{i,j} = 0$. Define $A: [n] \rightarrow [K]$ to be the group assigning function, i.e.\ $A(i) = \sum_{k=1}^K k \bI\{u_{i,k} \neq 0\}$, then
    \begin{align*}
        UV^\top &= \begin{bmatrix}
        \dfrac{1}{\sqrt{\abs{G_{A(i)}}}}v_{j,A(i)}
        \end{bmatrix}\\
        &=\begin{bmatrix}
        \dfrac{1}{\sqrt{\abs{G_{A(1)}}}}v_{1,A(1)} & \dfrac{1}{\sqrt{\abs{G_{A(1)}}}}v_{2,A(1)} & \cdots & \dfrac{1}{\sqrt{\abs{G_{A(1)}}}}v_{n,A(1)}\\
        \dfrac{1}{\sqrt{\abs{G_{A(2)}}}}v_{1,A(2)} & \ddots & & \vdots\\
        \vdots & & \ddots & \vdots \\
        \dfrac{1}{\sqrt{\abs{G_{A(n)}}}}v_{1,A(n)} & \cdots & \cdots & \dfrac{1}{\sqrt{\abs{G_{A(n)}}}}v_{n,A(n)}
        \end{bmatrix}.
    \end{align*}

    Observe that
    \begin{enumerate}[(i)]
    \item $A(i) = A(j) \iff i \in G_{A(j)} \iff j \in G_{A(i)} $
    
    \item $i \notin G_{A(j)} \iff (u_{i,A(j)} = 0 \land u_{j,A(i)} = 0)\implies (v_{i,A(j)} \ge 0 \land v_{j,A(i)} \ge 0)$
    
    \item $\inner{U}{V} = 0 \iff \forall j \in [n],\;\sum_{i \in G_{A(j)}}v_{i,A(j)} = 0$ .
    \end{enumerate}

    Denote $\bm{w} \coloneqq (UV^\top + VU^\top) \vone$, then
    \begin{align*}
        w_j &= 
        \sum_{i \in [n]} \mleft( \frac{1}{\sqrt{\abs{G_{A(j)}}}} v_{i,A(j)} +  \frac{1}{\sqrt{\abs{G_{A(i)}}}} v_{j,A(i)} \mright )\\
        &=
        \sum_{i \in G_{A(j)}} \mleft ( \frac{1}{\sqrt{\abs{G_{A(j)}}}} v_{i,A(j)} + \frac{1}{\sqrt{\abs{G_{A(i)}}}} v_{j,A(i)}\mright )\\
        &\quad+ \sum_{i \notin G_{A(j)}} \mleft ( \frac{1}{\sqrt{\abs{G_{A(j)}}}} v_{i,A(j)} + \frac{1}{\sqrt{\abs{G_{A(i)}}}} v_{j,A(i)}\mright ).
    \end{align*}

    By (i) and (iii), we know that 
    \begin{equation*}
         \sum_{i \in G_{A(j)}} \mleft ( \frac{1}{\sqrt{\abs{G_{A(j)}}}} v_{i,A(j)} + \frac{1}{\sqrt{\abs{G_{A(i)}}}} v_{j,A(i)}\mright ) = \sqrt{\abs{G_{A(j)}}} v_{j,A(j)}.
    \end{equation*}
    By (ii), we know that 
    \begin{equation*}
        \sum_{i \notin G_{A(j)}} \mleft(\frac{1}{\sqrt{\abs{G_{A(j)}}}} v_{i,A(j)} + \frac{1}{\sqrt{\abs{G_{A(i)}}}} v_{j,A(i)}\mright)\geq 0.
    \end{equation*}
    
    Thus we can write
    \begin{equation*}
        w_j= \sqrt{\abs*{G_{A(j)}}} v_{j,A(j)} + R_j\quad\text{for some } R_j \ge 0.
    \end{equation*}
    Next, we use proof by contradiction. Suppose there exists $V \in \rT_U\cM \cap \cC_U$ such that $V \neq 0$, then there must be both positive and negative entries in $V$ to satisfy $\inner{U}{V}= 0$. This implies that there exists $j_1 \in [n]$ such that $v_{j_1,A(j_1)} < 0$ since $v_{j,A(j)}$'s are the only entries that can take negative value. To satisfy $\inner{U}{V}= 0$, there must exist $j_2 \in G_{A(j_1)}$ such that $v_{j_2,A(j_1)} >0$. Then for such $j_2$,
    \begin{equation*}
        w_{j_2} = \sqrt{\abs*{G_{A(j_2)}}} v_{j_2,A(j_2)} + R_{j_2} 
        = \sqrt{\abs*{G_{A(j_2)}}} v_{j_2,A(j_1)} + R_{j_2} > 0.
    \end{equation*}
    This contradicts $(UV^\top + VU^\top) \vone = \vzero$. Therefore, $\rT_U\cM \cap \cC_U = \{0\}$.
\end{proof}
\begin{lem}
\label{lem:LICQ_submanifold_Kmeans}
If $K\ge2$, then the set $\cM$ defined in (\ref{eqn:our_Kmeans_manifold}) satisfies
linear independence constraint qualification (LICQ) for all $U\in\bR^{n\times r}$, and is therefore a smooth submanifold
of $\bR^{n\times r}$. 
\end{lem}
\begin{proof}[Proof of Lemma~\ref{lem:LICQ_submanifold_Kmeans}]
Following~\citet{boumal2020deterministic}, the set $\cM=\{U\in\bR^{n\times r}:\inner{A_{i}}{UU^\top }=b_{i}\text{ for all }i\}$
is a smooth submanifold of $\bR^{n\times r}$ if LICQ holds for all
$U\in\cM$, i.e., that $\sum y_{i}A_{i}U=0$ if and only $y=0$. For
the definition in~(\ref{eqn:our_Kmeans_manifold}), we can verify that
\begin{align*}
\norm*{\sum y_{i}A_{i}U}_{F}^{2} & =\norm{(\vone y^\top +y\vone^\top +y_{0}I)U}_{F}^{2}\\
 & =\norm{(\vone y^\top +y\vone^\top )U}_{F}^{2}+2\inner{(\vone y^\top +y\vone^\top )U}{y_{0}U}+y_{0}^{2}\norm{U}_{F}^{2}\\
 & =[ny^\top (I+UU^\top )y+2(\vone^\top y)^{2}]+4y_{0}(\vone^\top y)+y_{0}^{2}K\\
 & \ge2[(\vone^\top y)^{2}+2y_{0}(\vone^\top y)+y_{0}^{2}]+n\norm{y}^{2}+(K-2)y_{0}^{2}\\
 & =2(\vone^\top y+y_{0})^{2}+n\norm{y}^{2}+(K-2)y_{0}^{2} \\
 & \ge 2(\vone^\top y+y_{0})^{2}+n\norm{y}^{2}.
\end{align*}
So if $\sum y_{i}A_{i}U=0$, then from the last line we conclude that $\vone^\top y+y_{0}=0$ and $y=0$, so $(y,y_0)=0$ as claimed. The third line above is because $\inner{(\vone y^\top +y\vone^\top )U}U=2\inner y{UU^\top \vone}=2(y^\top \vone)$
and
\begin{align*}
\norm*{(\vone y^\top +y\vone^\top )U}_{F}^{2} & =\tr[(\vone y^\top +y\vone^\top )UU^\top (\vone y^\top +y\vone^\top )]\\
 & =\tr[\vone y^\top UU^\top y\vone^\top +y\vone^\top UU^\top \vone y^\top +2\cdot\vone y^\top UU^\top \vone y^\top ]\\
 & =\tr[\vone y^\top UU^\top y\vone^\top +y\vone^\top \vone y^\top +2\cdot\vone y^\top \vone y^\top ]\\
 & =ny^\top UUy+ny^\top y+2(\vone^\top y)^{2}.
\end{align*}
\end{proof}

\begin{lem}[Interior point construction for $\cM_r$]
    \label{lem:analytic_init}
    Given a $K \in \bN$, for any $r$ such that $r > K$, for large enough $n$, we have the following two cases:\\
    \textbf{Case 1}: $n \equiv 0 \pmod{r}$\\
    Denote $q \coloneqq n/r$, let $U_0 = (x-y)I + y \mathbf{11}^\top$, where
    \begin{equation*}
        x = \frac{1}{r}\mleft(1+\sqrt{(r-1)(K-1)}\mright), \qquad y = \frac{1}{r}\mleft(\sqrt{r-1}-\sqrt{K-1}\mright).
    \end{equation*}
    Then $U =(1/\sqrt{q}) \vone_q \otimes U_0$ is an interior point of $\cM_r$.\\
    \textbf{Case 2}: $n \not\equiv 0 \pmod{r}$\\
    Let us denote $q \coloneqq \lfloor n/r \rfloor$ and $p \coloneqq n \mod r$. Construct the block matrix $B \in \bR^{r \times r}$:
    \begin{equation*}
        B = 
        \begin{bmatrix}
            B_{1,1} & B_{1,2}\\
            B_{2,1} & B_{2,2}
        \end{bmatrix},
    \end{equation*}
where
    \begin{align*}
        B_{1,1} &= (x-y)I + y \vone_p\vone_p^\top, &\quad 
        B_{1,2} &= z \vone_p\vone_{r-p}^\top,\\
        B_{2,2} &= (w-z)I + z \vone_{r-p}\vone_{r-p}^\top, &\quad 
        B_{2,1} &= y \vone_{r-p}\vone_p^\top,
    \end{align*}
    The coefficients $x,y,z$ and $w$ depends on $n$, $K$ and $r$. They will be specified in the proof. Then
    \begin{equation*}
        U = 
        \begin{bmatrix}
                I_n & \vzero
            \end{bmatrix}(\vone_q \otimes B)
    \end{equation*}
    is an interior point of $\cM_r$.
\end{lem}
\begin{proof}[Proof of Lemma~\ref{lem:analytic_init}]
    For a general large enough $n$, $n$ is either divisible or nondivisible by $r$. We present two different constructions of an interior point of $\cM_r$ corresponding to the two cases.\\
    \textbf{Case 1}: $n \equiv 0 \pmod{r}$\\
    We first construct a $U_0 \in \bR^{r \times r}$ such that $U_0U_0^\top \mathbf{1}_r = \mathbf{1}_r$, and $\norm{U_0}^2_F = K$. Using the ansatz $U_0 = (x-y)I + y \mathbf{11}^\top$, where $x,y>0$, we can find $x$ and $y$ by solving the system:
    \begin{equation*}
    \begin{dcases}
        x + (r-1)y = 1 &(U_0U_0^\top \vone_r = \vone_r)\\
        x^2 + (r-1)y^2 =K/r &(\norm{U_0}^2_F = K)
    \end{dcases}.
    \end{equation*}
    The first equation gives $x = 1-(r-1)y$. By substituting into the second equation, we obtain the following quadratic equation of $y$:
    \begin{equation*}
        r(r-1)y^2 - 2(r-1)y + 1 - \frac{K}{r} = 0.
    \end{equation*}
    By the quadratic formula and $x,y>0$, we have the following solution
    \[\begin{dcases}
        x =1-(r-1)y = \frac{1}{r}(1+\sqrt{(r-1)(K-1)}), \\
        y =\frac{r-1-\sqrt{(r-1)(K-1)}}{r(r-1)}= \frac{1}{r}(\sqrt{r-1}-\sqrt{K-1}).
    \end{dcases}\]
    Note that if $r = K$, we can still solve the quadratic equation, but without an all positive solution.
    Denote $q \coloneqq n/r$, then $U=(1/\sqrt{q})\vone_q \otimes U_0$ is an interior point of $\cM_r$.\\
    \textbf{Case 2}: $n \not\equiv 0 \pmod{r}$\\
    Let us denote $q \coloneqq \lfloor n/r \rfloor$ and $p \coloneqq n \mod r$. We consider the ansatz
    \begin{equation*}
        U = 
        \begin{bmatrix}
                I_n & \vzero
            \end{bmatrix}(\vone_q \otimes B)
        \quad\text{for some block matrix }
        B = 
        \begin{bmatrix}
            B_{1,1} & B_{1,2}\\
            B_{2,1} & B_{2,2}
        \end{bmatrix}
        \in \bR^{r \times r},
    \end{equation*}
    where
    \begin{align*}
        B_{1,1} &= (x-y)I + y \vone_p\vone_p^\top, 
        &B_{1,2} &= z \vone_p\vone_{r-p}^\top,\\
        B_{2,2} &= (w-z)I + z \vone_{r-p}\vone_{r-p}^\top, 
        &B_{2,1} &= y \vone_{r-p}\vone_p^\top.
    \end{align*}
    Additional to the constraints $\|U\|^2_F = K$, and $UU^\top = \vone_n$, we assume that $U\vone_r = c_r\vone_n$, $U^\top \vone_n = c_c\vone_r$, and $c_rc_c = 1$ for some $c_r$ and $c_c$, which are sufficient for $UU^\top = \vone_n$. Then we can find $x$, $y$, $z$, and $w$ by solving the system:
    \begin{equation}
    \label{eq:xyzw_sys}
    \begin{dcases}
        x + (p-1)y + (r-p)z = py + w + (r-p-1)z,\\
        (q+1)x + q(r-1)y + (p-1)y = qw + q(r-1)z + pz, \\
        (py + w +(r-p-1)z)(qw + q(r-1)z + pz) = 1, \\
        (q+1)p(x^2 + (p-1)y^2 + (r-p)z^2) + q(r-p)(py^2 + w^2 +(r-p-1)z^2)= K.
    \end{dcases}
    \end{equation}
    The four equations correspond to the following constraints, respectively: $U\vone_r = c_r\vone_n$, $U^\top \vone_n = c_c\vone_r$, $c_rc_c = 1$, and $\|U_0\|^2_F = K$.
    From the first two equations, we can express $x$ and $y$ in terms of $z$ and $w$ (note that $n = qr + p$):
    \begin{equation}\label{eq:xy_in_zw}
        x =\mleft(1-\frac{1}{n}\mright)w + \frac{1}{n}z,
        \quad \quad
        y =\mleft(1+\frac{1}{n}\mright)z - \frac{1}{n}w.
    \end{equation}
    By substituting~(\ref{eq:xy_in_zw}) to the third and fourth equations of~(\ref{eq:xyzw_sys}), we are left with a system of quadratic equations of two variables:
    \begin{equation}
    \label{eq:zw_sys}
    \begin{dcases}
        a_1 z^2 + a_2 zw + a_3 w^2 + c_1 = 0,\\
        b_1 z^2 + b_2 zw + b_3 w^2 + c_2 = 0,
    \end{dcases}
    \end{equation}
    where
    \begin{align*}
        a_1 &= nr - qr + p - p(2q+1)/n, 
        &a_2 &= 2p(1 + 2q - n)/n,\\
        a_3 &= n - p(2q+1)/n,
        &b_1 &= (r+p/n-1)(n-q),\\
        b_2 &= (1-p/n)(n-q) + q(r+p/n-1),
        &b_3 &= q(1-p/n),\\
        c_1 &= -K,
        &c_2 &= -1.
    \end{align*}
    Now our goal is to solve~(\ref{eq:zw_sys}). By multiplying the first equation with $b_1$ and the second one with $a_1$ and subtraction, we can express $z$ in terms of $w$:
    \begin{equation}
    \label{eq:z_in_w}
        z = aw + \frac{b}{w}\quad\text{with }
        a = \frac{a_3 b_1 - a_1 b_3}{a_1 b_2 - a_2 b_1},\,
        b = \frac{c_1 b_1 - a_1 c_2}{a_1 b_2 - a_2 b_1}.
    \end{equation}
    Suppose that $w \neq 0$, we substitute~(\ref{eq:z_in_w}) into the second equation and multiply by $w^2$. The result is a quartic equation:
    \begin{equation}
    \label{eq:w_quin}
        (b_1 a^2 + b_2 a + b_3)w^4 + (2abb_1 +bb_2 + c_2)w^2 + b^2b_1 = 0.
    \end{equation}
    By solving~(\ref{eq:w_quin}) with $z=aw + b/w>0$, we obtain
    \begin{equation}
    \label{eq:w_sol}
        w = \sqrt{\frac{-(2abb_1 + bb_2 + c_2) 
        + \sqrt{(2abb_1 + bb_2 + c_2)^2 - 4 (a^2 b_1+ ab_2 + b3)b^2 b_1}}{2(a^2 b_1 + ab_2 + b_3)}}.
    \end{equation}
    The proof is completed by combining~(\ref{eq:xy_in_zw}),~(\ref{eq:z_in_w}), and~(\ref{eq:w_sol}). 
\end{proof}

\begin{proof}[Proof of \Cref{prop:factor}]
For ``$\supseteq$,'' if $U=\begin{bmatrix}\hat{\vone}_n & V\end{bmatrix}Q$,
then $UU^\top =n^{-1}\vone_n\vone_n^\top +VV^\top $ and hence $UU^\top \vone=\vone$
and $\tr(UU^\top)=K$ respectively, because $\vone_n^\top V=0$ and $\tr(VV^\top )=K-1$.
For ``$\subseteq$,'' let $U=P\Sigma Q$ denote the singular value
decomposition with $P^\top P=QQ^\top =I_{r}$. Since $UU^\top \vone_n=\vone_n$,
the decomposition can be chosen so that $P\Sigma e_{1}=(1/\sqrt{n})\vone_n=\hat{\vone}_n$.
So if $V=P\Sigma[e_{2},\dots,e_{r}]$, then $\vone_n^\top V=e_{1}^\top [e_{2},\dots,e_{r}]=0$
and $\norm{V}^{2}=\norm{U}^{2}-\norm{P\Sigma e_{1}}^{2}=K-1$.

In the final part, we first construct the inner approximation $S$
of the tangent space
\begin{align*}
\rT_{(V,Q)}\PMan & =\{(\dot{V},\dot{Q}):\vone_n^{\top}\dot{V}=0,\inner{V}{\dot{V}}=0,Q\dot{Q}^{\top}+\dot{Q}Q^{\top}=0\}\\
 & \supseteq\mleft\{(\dot{V},\dot{Q}):\vone_n^{\top}\dot{V}=0,\inner{V}{\dot{V}}=0,\dot{Q}Q^{\top}=\begin{bmatrix}\vzero & -h^{\top}\\
h & \vzero
\end{bmatrix}\mright\}=S.
\end{align*}
We observe that 
\[
\begin{bmatrix}
\hat{\vone}_n & V
\end{bmatrix}\dot{Q}=\begin{bmatrix}
\hat{\vone}_n & V
\end{bmatrix}\begin{bmatrix}\vzero & -h^{\top}\\
h & \vzero
\end{bmatrix}Q=
\begin{bmatrix}
Vh & \hat{\vone}_nh^{\top}
\end{bmatrix}Q
\]
and therefore the Jacobian operator is injective for all $(\dot{V},\dot{Q})\in S$:
\begin{align*}
\norm*{\Diff\varphi(V,Q)[\dot{V},\dot{Q}]}^{2} & =\norm*{\begin{bmatrix}\vzero & \dot{V}\end{bmatrix}Q+\begin{bmatrix}
\hat{\vone}_n & V
\end{bmatrix}\dot{Q}}^{2}\\
 & =\norm*{\begin{bmatrix}\vzero & \dot{V}\end{bmatrix}}^{2}+\norm*{\begin{bmatrix}Vh & \hat{\vone}_nh^{\top}\end{bmatrix}}^{2}
 \ge\norm{\dot{V}}^{2}+\norm{h}^{2}\ge\frac{1}{\sqrt{2}}\norm*{(\dot{V},\dot{Q})}^{2}
\end{align*}
where we used the fact that $\vone_n^{\top}\dot{V}=0$. Hence, the Jacobian operator is surjective, as claimed:
\[
\dim\bigl(\operatorname{image}(\Diff\varphi(V,Q))\bigr)\ge\dim(S)=n(r-1)-1=\dim(\rT_{U}\mathcal{M}).
\]
\end{proof}


\section{Additional Details on Riemannian Formulation}
\subsection{Equivalence between Riemannian Optimization and SQP}\label{subsec:Riem_SQP}
Let us consider a constraint optimization problem,
\begin{mini}
{\scriptstyle u\in\bR^n}{f(u),}{}{}
\addConstraint{g_i(u)}{= 0\label{eq:const_prob}}{\;\forall i.}
\end{mini}
where LICQ:
\[
\sum_{i}y_i \nabla g_i(u) = 0 \iff y_i=0\;\forall i.
\]
holds for every $u$ in the constraint set. Then $\cM\coloneqq \{u \in \bR^n: g_i(u) = 0\;\forall i\}$ is a smooth embedded manifold~\citep{boumal2023intromanifolds}, and we may employ Riemmanian methods to solve~(\ref{eq:const_prob}).

The Riemmanian method solves $\min_{u \in \cM} f(u)$ by solving 
\begin{align}\label{eq:R_Newt}
\min_{\dot{u}\in\rT_u\cM} f(u) + \inner{\grad f(u)}{\dot{u}} + \frac{1}{2} \inner{\Hess f(u)[\dot{u}]}{\dot{u}} + \frac{L}{6}\norm{\dot{u}}^3
\end{align}
at each iteration. One can show that this is equivalent to the SQP method that solve~(\ref{eq:const_prob}) by minimizing the Lagrangian
\begin{mini}
{\scriptstyle\dot{u}}{\cL(u, y^{(u)})+\inner{\nabla_u\cL(u,y^{(u)})}{\dot{u}}+\frac{1}{2}\inner{\nabla_{uu}^2\cL(u,y^{(u)})[\dot{u}]}{\dot{u}}+\frac{L}{6}\norm{\dot{u}}^3,}{}{}
\addConstraint{\inner{\nabla g_i(u)}{\dot{u}}}{= 0\label{eq:SQP}}{\;\forall i.}
\end{mini}
where
\[
\cL(u, y^{(u)})= f(u)+\sum_{i}y^{(u)}_{i} g_i(u)
\quad\text{and}\quad y^{(u)} = \argmin_y\norm{\nabla_u\cL(u, y)}.\]

Hence, our contribution is to efficiently solve the SQP subproblem by exploiting a block-diagonal-plus-low-rank structure in the Hessian, and the fact that there are only $r + r(r+1)/2 \ll n$ constraints. We provide more details in Appendix~\ref{subsec:deriv} and~\Cref{subsec:RgRH}.

To establish the equivalence, we first observe that the search space of~(\ref{eq:R_Newt}) and~(\ref{eq:SQP}) are the same. Indeed, the tangent space is $\rT_u\cM = \{\dot{u}:\inner{\nabla g_i(u)}{\dot{u}} = 0\;\forall i\}$. Next, we write the expressions for the Riemannian gradient and Hessian \citep[Prop.~3.61 and Cor.~5.16]{boumal2023intromanifolds}:
\begin{equation}\label{eqdef:Rg,RH}
    \grad f(u)\coloneqq\Proj_u\bigl(\nabla f(u)\bigr), \qquad \Hess f(u)[\dot{u}] \coloneqq \Proj_u\bigl(\Diff_u \grad f(u)[\dot{u}]\bigr),
\end{equation}
where $\Proj_u(v)\coloneqq\argmin_{\dot{u}\in\rT_{u}\cM}\norm{v-\dot{u}}$, and $D_u$ denotes the usual differential operator. We then obtain
\begin{equation}\label{eq:gradf}
\grad f(u) \coloneqq \Proj_u\bigl(\nabla f(u)\bigr) = \nabla f(u) + \sum_{i}y^{(u)}_{i} \nabla g_i(u) = \nabla_u\cL(u,y^{(u)}).
\end{equation}
For the second equality, see Equation (7.75) in~\citet{boumal2023intromanifolds}.

For the second order terms, we have $\inner{\Hess f(u)[\dot{u}]}{\dot{u}} = \inner{\nabla_{uu}^2\cL(u)[\dot{u}]}{\dot{u}}$ for all $\dot{u}$ in $\rT_u\cM$ by the facts that
\begin{enumerate}[(i)]
    \item $\Hess f(u)[\dot{u}] = \Proj_u\bigl(\nabla_{uu}^2\cL(u, y^{(u)})[\dot{u}]\bigr)$.
    \item $\inner{\Proj_u(v)}{\dot{u}} = \inner{v}{\dot{u}}$.
\end{enumerate}
We obtain fact (i) by~(\ref{eqdef:Rg,RH}) and~(\ref{eq:gradf}):
\begin{align*}
    \Hess f(u)[\dot{u}] \coloneqq& \Proj_u\bigl(\Diff\grad f(u)[\dot{u}]\bigr)\\
    = &\Proj_u\bigl(\nabla^2 f(u)[\dot{u}]+\sum_{i}(\Diff_u y^{(u)}_{i}) \nabla g_i(u)[\dot{u}]+ \sum_{i}y^{(u)}_{i} \nabla^2 g_i(u)[\dot{u}]\bigr)\\
    = &\Proj_u\bigl(\nabla^2 f(u)[\dot{u}]+ \sum_{i}y^{(u)}_{i} \nabla^2 g_i(u)[\dot{u}]\bigr)\\
    = &\Proj_u\bigl(\nabla_{uu}^2\cL(u, y^{(u)})[\dot{u}]\bigr).
\end{align*}
Inside the projection operator, the term $\sum_{i}(Dy^{(u)}_{i}) \nabla g_i(u)$ vanishes because $\rT_u\cM$ is a linear subspace, and $(\rT_u\cM)^\perp = \operatorname{span}\bigl(\nabla g_i(u)\bigr)$. Fact (ii) is due to that the projection operator is self-adjoint and that the projection of any tangent vector is itself. For more on the connection between SQP and Riemannian Newton method, we refer to~\citet{absil2009all, mishra2016riemannian}.
\subsection{General form of the Riemannian gradient and Hessian}\label{subsec:deriv}

Throughout this section, we use a bar over a function defined on the manifold $\cM$ to denote its smooth extension defined on a neighborhood of $\cM$ so that the Euclidean gradient and Hessian can be defined on $\cM$. Namely, for $f:\cM \rightarrow \bR^m$, we use $\bar{f}:N(\cM) \rightarrow \bR^m$ to denote the smooth extension. The notations $\grad f$ and $\Hess f$ denote the Riemmanian gradient and Hessian; $\nabla \bar{f}$ and $\nabla^2 \bar{f}$ denote the Euclidean gradient and Hessian.

For manifolds that can be defined by
\[
\min_{U\in\cM}f(U),\qquad\cM=\{U\in\bR^{n\times r}:\cA(UU^\top)+\cB(U)=c\},
\]
where $\cA\colon\bR^{n\times n}\to\bR^{m}$ and $\cB\colon\bR^{n\times r}\to\bR^{m}$ are linear operators, and $c\in\bR^{m}$, its tangent space can be written as:
\[
\rT_U \cM = \{\dot{U}\in\bR^{n\times r}:\cA(\dot{U}U^\top+U\dot{U}^\top) +\cB(\dot{U}) = 0 \}.
\]
We call the function $\cA(UU^\top)+\cB(U)=c$ as the \textit{defining function} of $\cM$.
Let us denote 
\begin{equation}\label{eq:L}
    L(\dot{U})\coloneqq \cA(\dot{U}U^\top+U\dot{U}^\top) +\cB(\dot{U}).
\end{equation} 
Immediately, we see that $\rT_U\cM = \ker(L)$, and the adjoint operator 
\begin{equation}\label{eq:L*}
    L^*(y) = 2 \cA^\top (y)U + \cB^\top(y).
\end{equation} The projection operator onto the tangent
space is defined to be \[\Proj_{U}(W)\coloneqq\argmin_{\dot{U}\in\rT_{U}\cM}\norm{W-\dot{U}}.\]
We also know that $\ker(L)^\perp = \operatorname{image}(L^*)$. Thus, by the orthogonal projection theorem, we may see that 
\begin{equation}\label{eq:orth_dec}
    \Proj_{U}(W) = W - W^\perp = W - L^*(\tilde{y}),
\end{equation}
where $\tilde{y} = \argmin_{y}\|W-L^*(y)\|$ is the solution to the linear system $W-L^*(y) \in \ker(L) =\rT_U\cM$, and both $\Proj_U(W)$ and $\tilde{y}$ are unique.

Consequently, the Riemannian gradient and the Hessian matrix-vector product have the following form:
\begin{align}\label{eq:RgRH}
    \begin{dcases}
\grad f(U)\coloneqq\Proj_U(\nabla \bar{f}(U)) = \nabla \bar{f}(U)+2[\cA^\top (y_{U})]U+\cB^\top (y_{U}),\\
\Hess f(U)[\dot{U}]\coloneqq\Proj_U\bigl(\Diff\grad f(U)[\dot{U}]\bigr) =\Proj_{U}\bigl(\nabla^{2}\bar{f}(U)[\dot{U}]+2[\cA^\top (y_{U})]\dot{U}\bigr),
\end{dcases}
\end{align}
where $y_{U} = -\tilde{y}$ in~(\ref{eq:orth_dec}) is the unique Lagrange multipliers 
\begin{equation}\label{eq:y_U}
    y_{U}=\argmin_{y\in\bR^{m}}\norm*{\nabla f(U)+2[\cA^\top (y)]U+\cB^\top (y)}.
\end{equation}
For a detailed proof, see \citet{boumal2020deterministic}. 

\subsection{Efficient computation of the Riemannian gradient and Hessian}\label{subsec:RgRH}

We first write down the Euclidean gradient and Hessian for our objective function, and then explain how to compute the Riemannian counterparts efficiently. Specifically, we show that $y_U$ can be computed in $O(nr+r^3)$ time.

We decompose the objective function as $f = f_1 + \mu f_2 $, where \[f_{1}(V,Q)\coloneqq\inner C{VV^\top },\qquad f_{2}(V,Q)\coloneqq-\sum_{i,j}\log\varphi_{i,j}(V,Q).\] For $f_1$, the Euclidean gradients are \(\nabla_{V}f_{1}(V,Q)=2CV\) and \(\nabla_{Q}f_{1}(V,Q)=0\), respectively.
For $f_2$, the gradients are
\[
\nabla_{V}f_{2}(V,Q)=-U_{(-1)}\hat{Q}^\top,\qquad
\nabla_{Q}f_{2}(V,Q)=-\hat{V}^\top U_{(-1)},\]
where $U\coloneqq\varphi(V,Q)$, and $U_{(-1)}$ is the elementwise reciprocal, i.e., $[U_{(-1)}]_{i,j}=U_{i,j}^{-1}$. For convenience, we also define $\hat{Q}\coloneqq DQ$ with $D\coloneqq \begin{bmatrix}\vzero_{r-1} &
I_{r-1}
\end{bmatrix}$. Note that \(\hat{Q}\) is simply \(Q\) without the first row. Putting these together, we obtain \[G_V \coloneqq \nabla_V f(V,Q)=2CV-\mu U_{(-1)}\hat{Q}^\top,\qquad G_Q \coloneqq \nabla_Q f(V,Q)=-\mu\hat{V}^\top U_{(-1)}.\]

The Euclidean Hessians can be given in vectorized form as
\[
\nabla^{2}f_{1}=\begin{bmatrix}2I_{r-1}\otimes C & \vzero\\
\vzero & \vzero
\end{bmatrix},\qquad
\nabla^{2}f_{2}=\begin{bmatrix}J_{V}^\top\\[.5em]
J_{Q}^\top
\end{bmatrix}\dvec\bigl(U_{(-2)}\bigr)\begin{bmatrix}J_{V}&
J_{Q}
\end{bmatrix}-\begin{bmatrix}\vzero & H_{VQ}\\
H_{QV} & \vzero
\end{bmatrix},
\]
where 
\[J_{V}= \hat{Q}^\top\otimes I_n,\qquad J_{Q}=I_{r}\otimes QU^\top,\] and \[H_{VQ}=\mleft(D\otimes U_{(-1)}\mright)K^{(r,r)} = H_{QV}^\top,\]
with \(U_{(-2)}\) being the elementwise square of \(U_{(-1)}\), \(\dvec(\cdot)\coloneqq\diag[\vect(\cdot)]\), and $K^{(n,r-1)},K^{(r,r)}$ denoting the commutation matrices \citep[Sec.~3.7]{magnus99}. We can then compute the Riemannian gradient and Hessian-vector-product according to~(\ref{eq:RgRH}) and~(\ref{eq:y_U}). In the remainder of this section, we show how to efficiently solve~(\ref{eq:y_U}).

The manifold we consider, $\widetilde{\cM}\coloneqq\cV \times \Orth(r)$, can be written as
\begin{align*}
    \widetilde{\cM}= \mleft\{(V,Q) \in \bR^{n\times (r-1) }  \times \bR^{r \times r}:\vone_n^{\top}V=\vzero_{r-1},\tr(VV^\top) = K-1,\svect(QQ^\top - I_r)=\vzero\mright\},
\end{align*}
where $\svect$ denotes the symmetric vectorization \citep[Appx.~E]{klerk2006}.
Note that there are no cross terms ($VQ^\top$ or $QV^\top$) in the defining functions of $\cM$. Thus, we can treat the defining functions with respect to $V$ and $Q$ separately. The corresponding terms are
\begin{align*}
\cA_V(VV^\top)\coloneqq\begin{bmatrix}
\vzero_{r-1} \\ \tr(VV^\top)
\end{bmatrix}, \qquad \cB_V(V)\coloneqq\begin{bmatrix}
\vone_n^{\top}V \\ 0
\end{bmatrix},\qquad \cA_Q(QQ^\top)\coloneqq\svect(QQ^\top)
\end{align*}
and 
\[c_V\coloneqq\begin{bmatrix}
\vzero_{r-1} \\ K-1
\end{bmatrix}\quad\text{and}\quad c_Q\coloneqq \svect(I_r).\]

Mimicking ~(\ref{eq:L}) and~(\ref{eq:L*}), we use the notation $L_V$, $L_Q$, $L_V^*$, and $L_Q^*$ respectively. For any $(y_1, y_2) \in \bR^{r-1} \times \bR$ and $y_3 \in \bR^{r(r+1)/2}$, we have
\[
L_V^*(y_1,y_2) = \vone_n y_1^\top +2y_2 V\quad\text{and}\quad L_Q^*(y_3) = 2 \smat(y_3)Q,
\]
where $\smat$ is the inverse of $\svect$, that is, $\smat(\svect(M)) = M$ for all symmetric matrices $M$. By solving the linear systems \[G_V - L_V^*(y_1, y_2) \in \ker(L_V)\quad\text{and}\quad G_Q - L_Q^*(y_3) \in \ker(L_Q),\]we obtain the following closed form solutions:
\begin{equation}\label{eq:L_multipliers}
    \tilde{y}_1 = \frac{1}{n} G_V^\top \vone_n,\qquad \tilde{y}_2 = \frac{\inner{G_V}{V}}{2(K-1)},\qquad \tilde{y}_3 = \frac{1}{4}\svect(G_Q Q^\top+QG_Q^\top). 
\end{equation}

The computation of $\tilde{y}_1$, $\tilde{y}_2$, and $\tilde{y}$ in total requires $O(nr+r^3)$ time. Therefore, we can compute $y_U = -(\tilde{y}_1,\tilde{y}_2,\tilde{y}_3)$ with the same cost. Given a Euclidean gradient and a Euclidean Hessian-vector product, we may write out explicitly the Riemannian gradient:
\begin{equation}\label{eqn:}
\grad f(V,Q) =\begin{bmatrix}
\mleft(I_n - \dfrac{1}{n}\vone_n\vone_n^\top\mright) G_V- \dfrac{\inner{G_V}{V}}{K-1} V & \dfrac{G_QQ^\top - QG_Q^\top}{2}Q
\end{bmatrix}
\end{equation}
and the Riemannian Hessian-vector product:
\begin{align*}
    \Hess f(V,Q)[\dot{V},\dot{Q}] =\begin{bmatrix}
    \Proj_V\mleft(\nabla^2f(V,Q)[\dot{V}, \dot{Q}]_V -\dfrac{\inner{G_V}{V}}{K-1}\dot{V}\mright)\\
    \Proj_Q\mleft(\nabla^2f(V,Q)[\dot{V}, \dot{Q}]_Q - \dfrac{G_QQ^\top + QG_Q^\top}{2}\dot{Q}\mright)       
    \end{bmatrix}^\top.
\end{align*}

\subsection{Feasible initial point}\label{subsec:feas}
In this section, we show that $r>K$ is necessary and sufficient for the existence of an interior point of $\cM$. The following Lemma~\ref{lem:feasibility} shows the necessity of $r>K$. When $r=K$, the structure of the unique $U\in\bR_{+}^{n\times K}$
in Lemma~\ref{lem:membership2clusterlabel} can be explicitly written as
\begin{equation}\label{eqn:group_assignment_mat}
    U=\mleft[\frac{1}{\sqrt{\abs{G_{1}}}}\vone_{G_{1}},\ \frac{1}{\sqrt{\abs{G_{2}}}}\vone_{G_{2}},\dotsc,\ \frac{1}{\sqrt{\abs{G_{K}}}}\vone_{G_{K}}\mright],
\end{equation}
where $\vone_{G_{k}}\in\{0,1\}^{n}$ denotes the binary vector with
its support being $G_{k}$.
\begin{lem}[Isolated feasibility when $r=K$]
\label{lem:feasibility} Let $\cM_{+}=\cM\cap\bR_{+}^{n\times K}$
and $\cM'_{+}=\cM'\cap\bR_{+}^{n\times K}$, where $\bR_{+}^{n\times K}=\{U\in\bR^{n\times K}:U\geq0\}$.
Then, we have: (i) $\cM_{+}=\cM'_{+}$; (ii) $U\in\cM_{+}$ if and only if $U$ is a group assignment matrix defined in~(\ref{eqn:group_assignment_mat}); (iii) if $U$ is a group assignment matrix, then the intersection of the
tangent space $\rT_{U}\cM$ and the cone $\cC_{U}\coloneqq\{V\in\bR^{n\times K}:v_{ij}\geq0,\forall u_{ij}=0\}$
is trivial, i.e., $\rT_{U}\cM\cap\cC_{U}=\{0\}$.
\end{lem}
In Lemma~\ref{lem:analytic_init}, we moreover provide a complete analytical construction for an interior point of $\cM$ when $r>K$. Here, we present the construction when $n = qr$ is an integer multiple of $r$. Let $U_0 = (x-y)I + y \vone_n\vone_n^\top$, where $x = r^{-1}(1+\sqrt{(r-1)(K-1)})$ and $y = r^{-1}(\sqrt{r-1}-\sqrt{K-1})$.

Then $U = \hat{\vone}_q \otimes U_0$ is an interior point of $\cM$. 
Next, we show how to compute the pair $V$, $Q$ corresponding to the interior point $U$ by SVD. For a given $U \in \cM_r$, let $U = P_U \Sigma Q_U^\top$ be the SVD of $U$. We can find $(V, Q)$ such that $U = \hat{V}Q$ by $V=\operatorname{sgn}({P_U}_{(1,1)}){[P_U\Sigma]}_{(:,2:r)}$ and $Q = \operatorname{sgn}({P_U}_{(1,1)})Q_U^\top$.

\subsection{Lipschitz continuity of penalty \label{appsec:Lipschitz}}
To apply the guarantees in \Cref{subsec:Riemann}, we need to take care of the logarithmic penalty in~(\ref{eq:manif2}) since it does not have Lipschitz gradients nor Hessians over its whole domain. The standard workaround, widely used in the analysis of nonlinear interior-point methods, is to observe that all iterates $U_{k}=\varphi(V_{k},Q_{k})$ remain strictly feasible.
Consequently, the penalty could be modified by a Huber-style smoothing, where $\delta=\min_{i,j,k}(U_{k})_{i,j}>0$:
\[
r(x)=\begin{cases}
\log x & x\ge\delta\\
\log\delta+\dfrac{(x-\delta)}{\delta}-\dfrac{(x-\delta)^{2}}{2\delta^{2}}+\dfrac{(x-\delta)^{3}}{2\delta^{3}} & x<\delta
\end{cases}
\]
The function $r(x)$ is both concave and has Lipschitz Hessians. Therefore, the guarantees in \Cref{subsec:Riemann} apply. The smoothing is only needed for theoretical purposes. In practice, we apply the Riemannian algorithms directly to $\log x$, and not to $r(x)$. Since we have assumed that all queries satisfy $x\ge\delta$, the actual behavior remains consistent with the smoothed model. 

\section{Efficient Implementation and Cost of Bisection Search}\label{appsec:proof_bisect}
To implement the proposed method, we vectorize the input as $u=(v, q)$, with $v\coloneqq\vect(V^\top)$ and $q\coloneqq\vect(Q)$. Since $C=-XX^\top$, we rewrite the cost function form \Cref{subsec:RgRH} as \[f(u)=-\norm{X^\top V}^2-\mu\vone_n^\top\log\bigl(\varphi(V,Q)\bigr)\vone_n,\] and define the constraint functions
\[g_1(u)=\vone_n^\top V,\qquad g_2(u)=\norm{V}^2-(K-1),\qquad g_3(u)=\operatorname{svec}(Q Q^\top-I_r).\] The Jacobian \(J\) and Hessian \(H\) are computed analytically, as in \Cref{subsec:RgRH}, in order to exploit their sparsity.

Since we have restricted ourselves to optimizing over $p \in \rT_u\widetilde{\cM}_r$, we may replace the Riemannian Hessian $H$ with $\tilde{H}\coloneqq\nabla_{uu}^2\cL(u, y_u)$. The key observation is that $\tilde{H}$ admits a block-diagonal-plus-low-rank structure.

For convenience, we list some of the derivatives in this section. The (Euclidean) Jacobian of the constraints can be written in block form as \[J=\begin{bmatrix}
J_{1v} & \bm{0}\\
J_{2v} & \bm{0}\\
\bm{0} & J_{3q}
\end{bmatrix},\] where \[J_{1v}=(I_{r-1}\otimes\vone_n^\top)K^{(r-1,n)},\qquad J_{2v}=2v,\qquad J_{3q}= S_r(I+K_{r,r})(Q \otimes I_r),\] where $S_r \in \bR^{r(r+1)/2 \times r^2}$ is the matrix satisfying $S\vect(Q) = \svect(Q)$. 

Computing the second-order derivatives of $g_1,g_2$ is straightforward, since $J_{1v}$ is constant in $V$ and $J_{2v}$ is linear. To compute the second-order derivatives of $g_3$, since we only need to compute $\sum_{j\in[r(r+1)/2]} \tilde{y}_{3,j} \nabla^2 g_{3,j}(q) $, we note that for any $y \in \bR^{r(r+1)/2}$, 
\[
J_{3q}^\top y = 2\vect(\smat(y)Q).
\]
Consequently, we have for any $y \in \bR^{r(r+1)/2}$,
\[
\sum_{j\in[r(r+1)/2]} y \nabla^2 g_{3,j} = 2I_r \otimes\smat(y).
\]

Collecting results, we have
\[\tilde{H}=\begin{pmatrix}
H_{vv} - BB^\top & H_{vq}\\
H_{qv}  & H_{qq}
\end{pmatrix},\]
with 
\begin{align}
B&=\sqrt{2}(X\otimes I_{r-1})\\
H_{vv}&=\mu (I_n\otimes \hat{Q})\dvec\mleft(U_{(-2)}^\top\mright)(I_n\otimes \hat{Q}^\top) +  2 \tilde{y}_2 I,\\
H_{qq}&=\mu (I_r\otimes\hat{V}^\top)\dvec\mleft(U_{(-2)}\mright)(I_r\otimes\hat{V}) + 2I_r \otimes\smat(\tilde{y}_3),
\end{align}
and
\begin{equation}
H_{vq}=-\mu (U_{(-1)}\otimes D)+ \mu(I_n\otimes\hat{Q})K^{(n,r)}\dvec\mleft(U_{(-2)}\mright)(I_r\otimes\hat{V}) =H_{qv}^\top,
\end{equation} where $\tilde{y_2}, \tilde{y}_3$ follow~(\ref{eq:L_multipliers}). 

To solve the saddle point problem arising from subproblem \ref{eqn:newton_cube} via bisection search, we solve the linear system
\[
\mleft[\begin{array}{c|cc}
H_{vv}+\lambda I-BB^\top & H_{vq}  & J_{v}^\top \\
\hline H_{qv} & H_{qq}+\lambda I & J_{q}^\top \rule{0pt}{2.0ex}\\
J_{v} & J_{q} & \vzero
\end{array}\mright]\mleft[\begin{array}{c}
\dot{v}\\
\hline\dot{q}\rule{0pt}{2.0ex}\\
r
\end{array}\mright]=\mleft[\begin{array}{c}
-g_{v}\\
\hline -g_{q}\rule{0pt}{2.0ex}\\
\vzero
\end{array}\mright].
\]
Repartitioning along the lines yields:
\[
\mleft[\begin{array}{c|c}
K_{11} & K_{12} \\\hline
K_{21} & K_{22} \rule{0pt}{2.0ex} 
\end{array}\mright]\mleft[\begin{array}{c}
x_{1}\\
\hline x_{2} \rule{0pt}{2.0ex}
\end{array}\mright]=\mleft[\begin{array}{c}
b_{1}\\
\hline b_{2} \rule{0pt}{2.0ex}
\end{array}\mright]
\]
with $x_{1},b_{1}\in\bR^{n(r-1)}$ and $x_{2},b_{2}\in\bR^{r^{2}+m}$.
In particular, we observe that the block $K_{11}$ has the form $K_{11}=D_{11}-BB^\top$,
where $D_{11}=H_{vv}+\lambda I$ is block-diagonal, with $n$ blocks of $r-1$, and
$B$ has at most $dr$ columns. Therefore, we instead solve
\[
\mleft[\begin{array}{c|cc}
D_{11} & K_{12}  & B\\
\hline K_{21} & K_{22} & \vzero \rule{0pt}{2.0ex}\\
B^\top  & \vzero & I
\end{array}\mright]\mleft[\begin{array}{c}
\dot{v}\\
\hline \dot{q} \rule{0pt}{2.0ex}\\
z
\end{array}\mright]=\mleft[\begin{array}{c}
b_{1}\\
\hline b_{2} \rule{0pt}{2.0ex}\\
\vzero
\end{array}\mright].
\]
First, it costs $n(r-1)^{3}=O(nr^{3})$ time to invert $D_{11}$.
Afterwards, forming and solving the size $m+r^{2}+rd$ Schur complement problem:
\begin{equation}\label{eqn:schur}
\mleft(L_{22}-L_{12}^\top D_{11}^{-1}L_{12}\mright)\begin{bmatrix}\dot{q}\\
z
\end{bmatrix}
=\begin{bmatrix}b_{2}\\
\vzero
\end{bmatrix}-L_{12}^\top D_{11}^{-1}b_{1},
\end{equation} where \[L_{12}\coloneqq\begin{bmatrix}K_{21} & B\end{bmatrix}\qquad L_{22}\coloneqq\begin{bmatrix}K_{22} & \vzero\\
\vzero & I
\end{bmatrix},\] cost $O(nr^{3}(d+r)+r^{6}+r^{3}d^{3})$ time. In the end, we substitute
to recover
\[
\dot{v}=D_{11}^{-1}\mleft(b_{1}-L^{12}\begin{bmatrix}\dot{q}\\
z
\end{bmatrix}\mright)
\]
in $O(nr^{3}(d+r))$ time, and apply retractions to $\dot{v}$ and $\dot{q}$. In total, it takes $O(nr^{3}(d+r)+r^{6}+r^{3}d^{3})$
time to solve the system, which is indeed $n\cdot\poly(r,d)$. Putting pieces together, a pseudo-code of our Riemannian method is shown in \Cref{alg:riemann_tr}.

\section{Additional Numerical Results}\label{appsec:add_num}
We collect additional numerical results in this section.

\paragraph{Dataset visualization.} \Cref{fig:gmm_vis} and \Cref{fig:cytof_vis} display the first two principal components of the GMM dataset and CyTOF dataset, respectively.

\begin{figure}[!htb]
    \centering
    \includegraphics[width=0.8\textwidth]{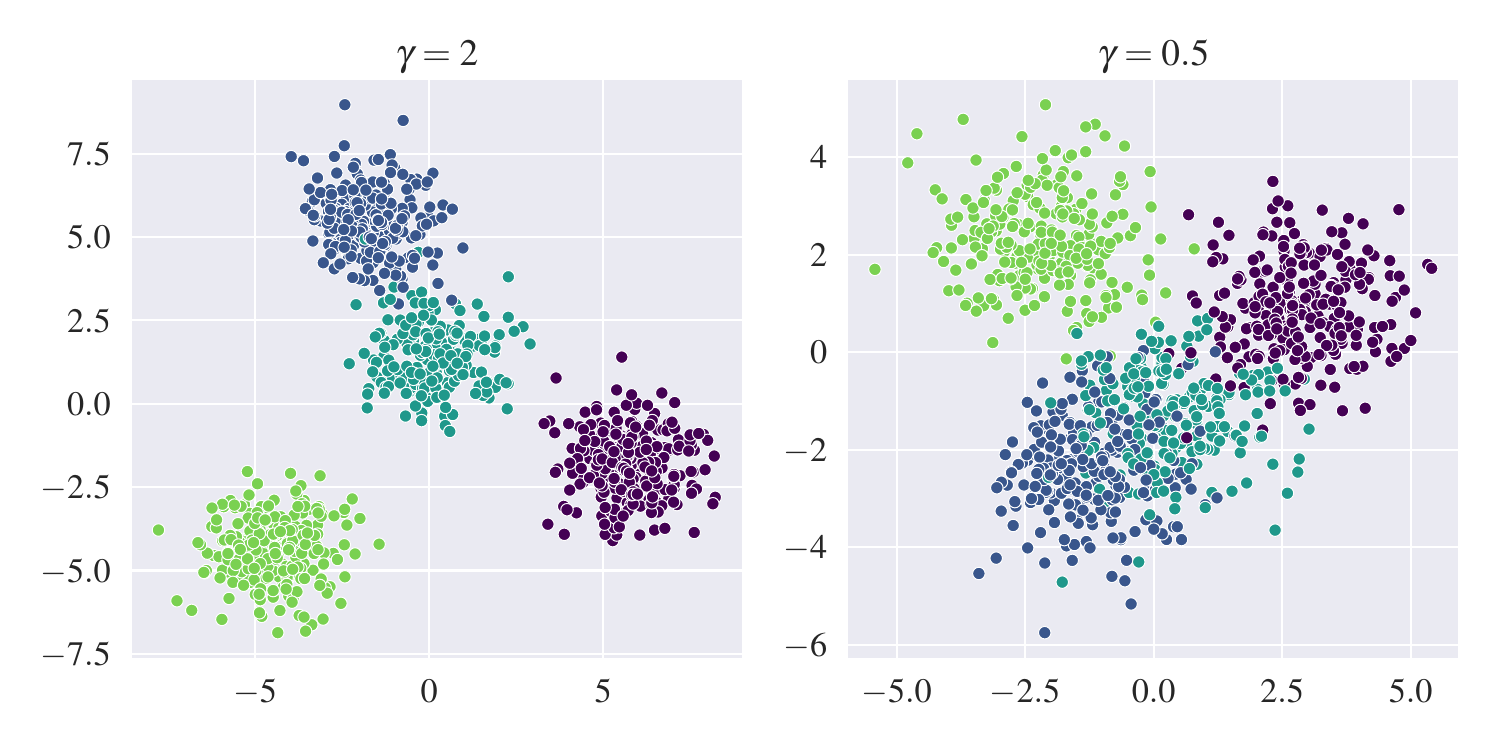}
    \caption{Visualizing the effect of the separation parameter in GMMs. As \(\gamma\) decreases, the clusters become increasingly difficult to distinguish.\label{fig:gmm_vis}}
\end{figure}

\begin{figure}[!htb]
    \centering
    \includegraphics[width=0.4\textwidth]{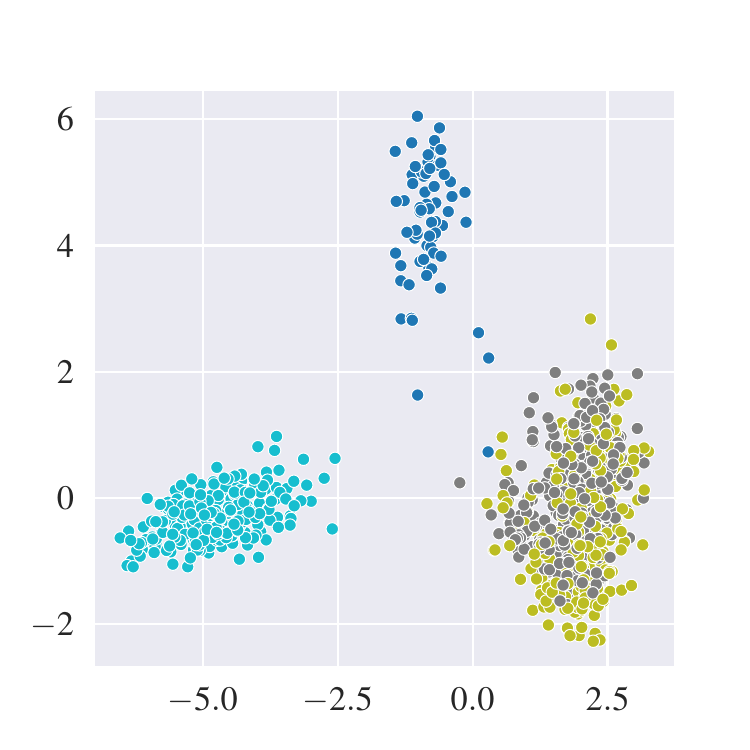}
    \caption{Visualization of the CyTOF dataset. Two clusters exhibit significant overlap, implying the difficulty of clustering.\label{fig:cytof_vis}}
\end{figure}

\paragraph{Benchmark on CIFAR10.}
\Cref{fig:cifar_result} shows the performance of our method on the CIFAR10 image classfication dataset. The original $32\times32$ color images were processed using a pre-trained Vision Transformer (ViT-B-16) \citep{dosovitskiy2021an} to extract 768-dimensional features, which were then reduced to $d=50$ using PCA. We use a pre-trained model to avoid breaching the unsupervised setting. From the dataset, 25000 images across five classes were selected; In each trial, we draw a subsample of 1,000 images and perform clustering. We repeated this procedure 50 times.
\begin{figure}[!htbp]
    \centering
    \includegraphics[width=0.8\textwidth]{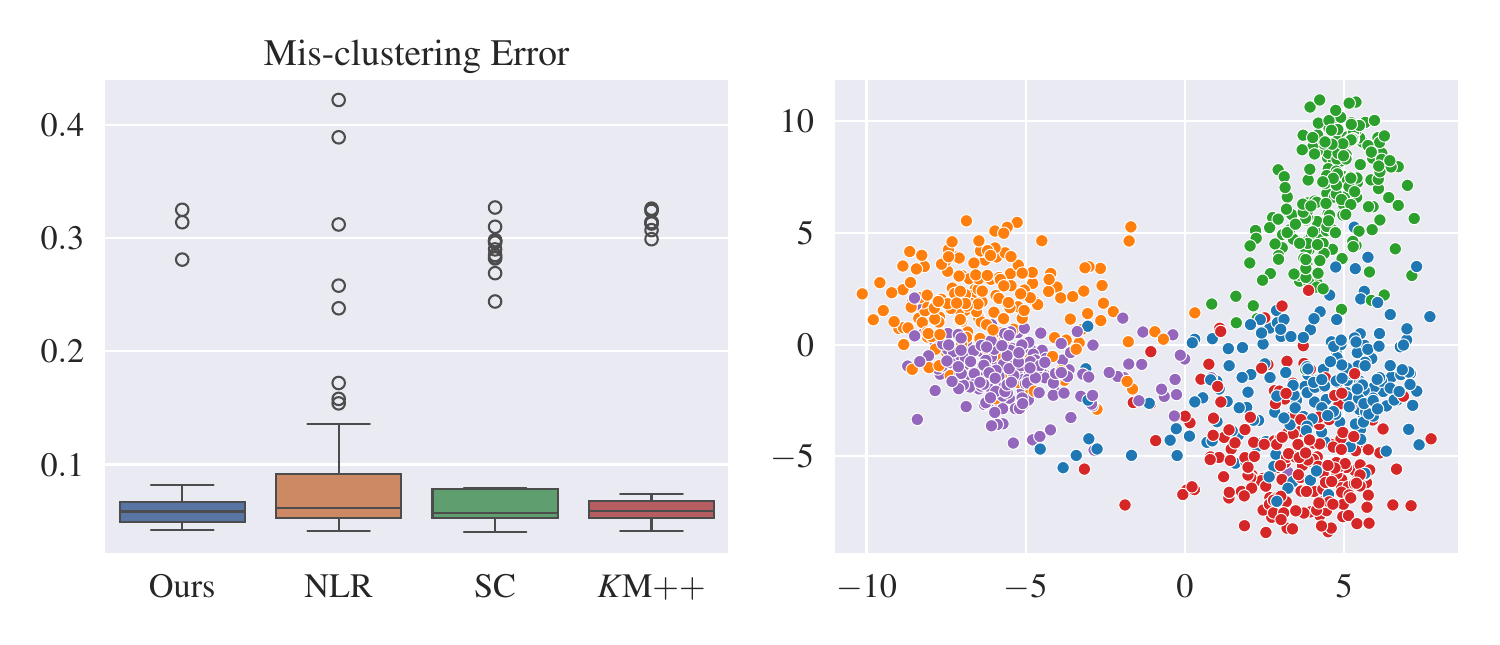}
    \caption{Real-world benchmark on the CIFAR10 data. (Left): comparison to other state-of-the-art methods. (Right): first two principal components of the image embeddings. \label{fig:cifar_result}}
\end{figure}

\paragraph{Generalization to kernelized $K$-mean.} Since the data enter the problem through the Gram matrix $C=-XX^\top$ in \Cref{eq:manif}, extending the implementation to the kernel $K$-means is straightforward. We adapted our implementation to accept a kernel matrix as input and conducted experiments using the RBF kernel on two toy datasets. The results are shown in \Cref{fig:rbf_result}.
\begin{figure*}[!htb]
    \centering
    \includegraphics[width=0.9\textwidth]{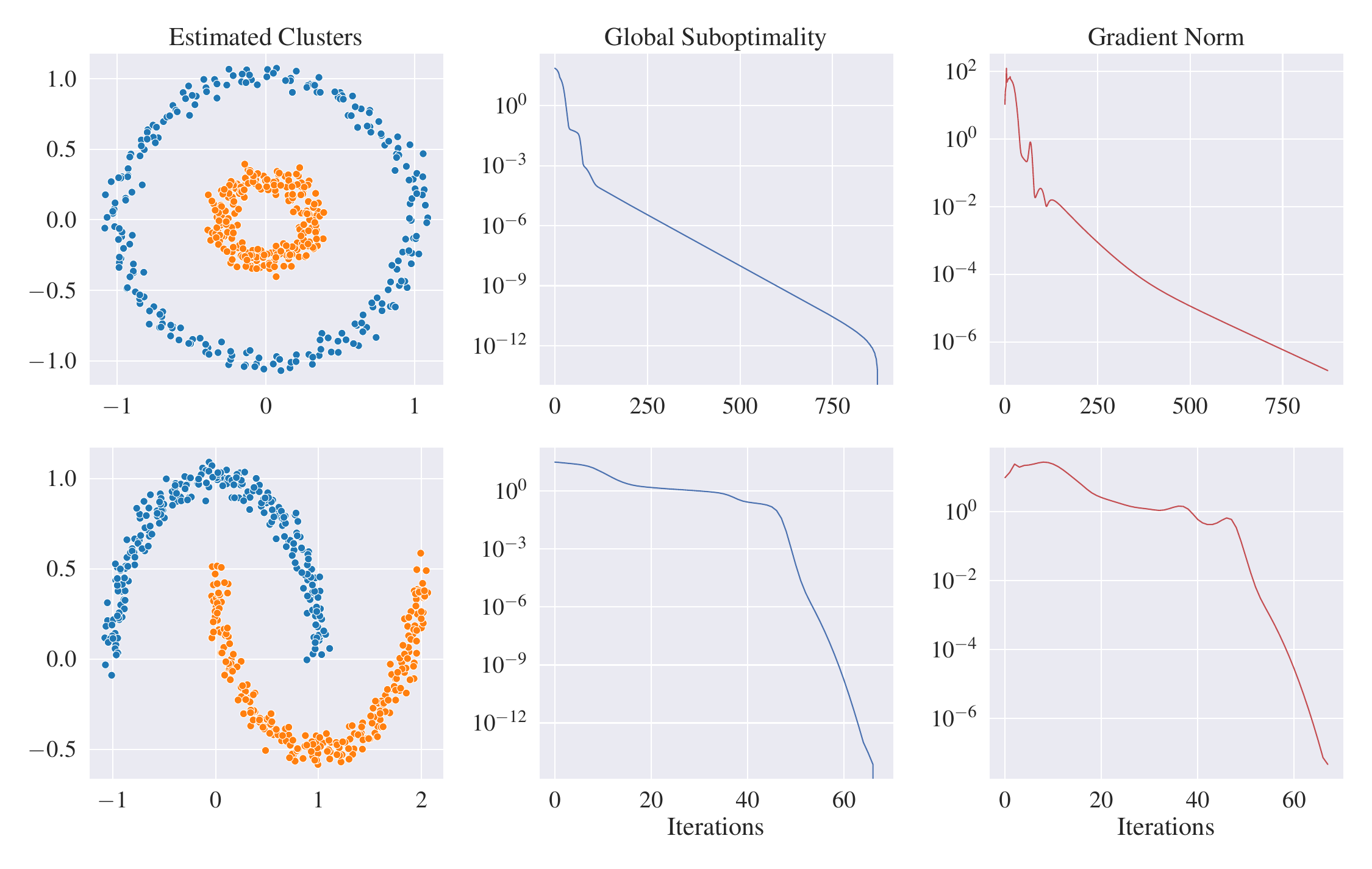}
    \caption{Clustering with the RBF kernel. Demonstrated on the Circles and Moon datasets from \texttt{scikit-learn}. The kernel bandwidths are 3 and 15, respectively. \label{fig:rbf_result}}
\end{figure*}

Nevertheless, without access to the factored matrix, we can no longer perform the sparse partition described in \Cref{appsec:proof_bisect}, and therefore we lose the linear-time guarantee for solving the subproblem. To extend the linear-time core claim would further require low-rank approximations of kernel matrices $C$, e.g.~via Nyström. We leave this as important future work outside the scope of this particular paper.

\paragraph{Impact of imbalanced clusters.}
\Cref{fig:perf_imb} shows the impact of imbalanced cluster sizes. In this example, the clusters contain approximately $10\%, 20\%, 20\%,$ and $50\%$ of the $n=1000$ data points. We compared our method with NLR and spectral clustering (SC), omitting $K$-means++ due to its substantially higher error, which would distort the scale of the results. As the imbalance among clusters increases, we noted that the convergence becomes less stable with the default parameters, so we increase the rank $r$ from $K+1$ to $K+3$. The new imbalanced cluster size experiment further reveals two things: (i) the performance of all methods under comparison degraded, and (ii) our method is the most robust (with the narrowest error distribution) among all.
\begin{figure}[!hbtp]
    \centering
    \includegraphics[width=\textwidth]{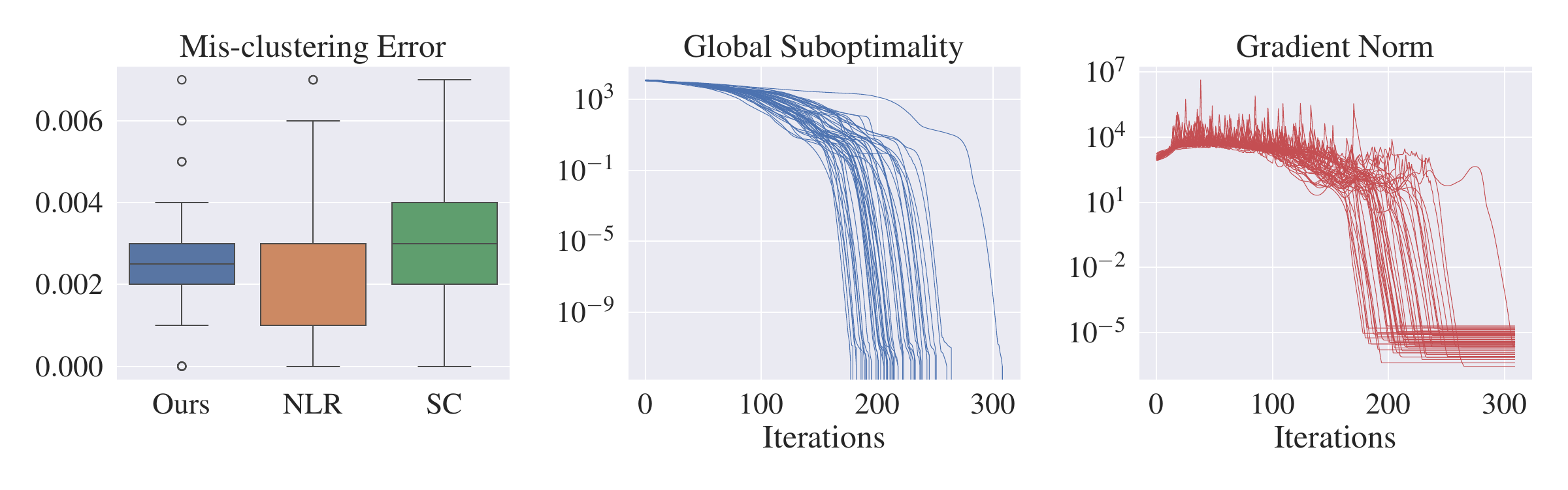}
    \caption{Clustering performance and convergence behavior when cluster sizes are imbalanced. \label{fig:perf_imb}}
\end{figure}

To explain this empirical observation, it is known that a similar (and more general) recovery threshold in the motivating $K$-means SDP to \Cref{eqn:recovery_threshold} holds for unbalanced cluster sizes. In particular, the exact recovery threshold depends on the minimum of harmonic means of the cluster sizes, i.e.~$m\coloneqq\min_{1 \leq k \neq l \leq K}(2 n_k n_l)/(n_k + n_l)$, in the following way \citep{chen2021cutoff}:
\[
\overline{\Theta}^2 \gtrsim  1 + \sqrt{1+ \frac{d}{m \log{n}}}.
\]
In the special case of $K$ equal clusters, $m = n / K$ attains the maximum possible value. Hence, when the cluster imbalance is significant, a larger centroid separation is required for the proposed algorithm to achieve clustering accuracy and convergence behavior comparable to that observed in the equal-cluster case.

\paragraph{Runtime scaling with respect to $r$ and $d$.} \Cref{fig:r_vs_d} shows the average per-iteration runtime as as $r$ and $d$ vary. The setting is the same as that of \Cref{fig:nlr_vs_tr}, with $n=1000$ fixed. The observed scalings for $r$ and $d$ are roughly $O(r^3)$ and $O(d)$, consistent with the leading terms of $O(nr^{3}(d+r)+r^{6}+r^{3}d^{3})$. Because $n$ dominates as a leading constant, the contributions of the $r^6$ and $r^3d^3$ terms are unlikely to become significant until $r$ and $d$ are extremely large.
\begin{figure}[!htb]
    \centering
    \includegraphics[width=0.8\textwidth]{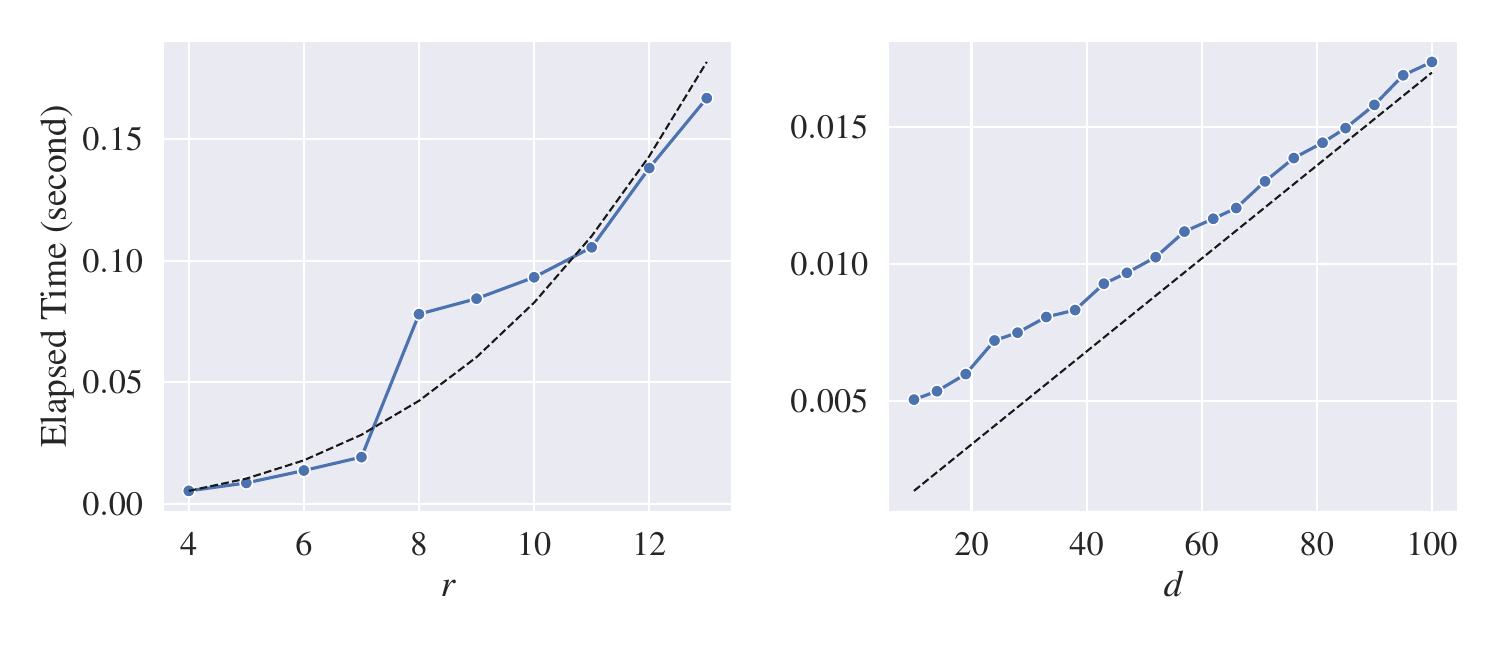}
    \caption{Average per-iteration runtime versus rank and dimension, $n=1000$. The measured runtimes exhibit approximately  $O(r^3)$ and $O(d)$ scaling (dashed lines). \label{fig:r_vs_d}}
\end{figure}

\paragraph{Robustness to initialization.} As illustrated by \Cref{fig:2opt}, our method is robust to initialization, all 50 trials successfully converged to second-order optimal solutions. Although the solutions differ (\Cref{fig:init_U}), their corresponding membership matrices $Z$ are close to each other (\Cref{fig:init_Z}), and yield identical clustering result. Moreover, the minimum eigenvalues upon convergence form distinct clusters that align with clusters in the recovered membership matrix $Z$, as shown in \Cref{fig:init_Z}. These local critical points consistently produce perfect clustering, indicating that they remain close to the global optimum.

\begin{figure*}[!htb]
    \centering
    \includegraphics[width=\textwidth]{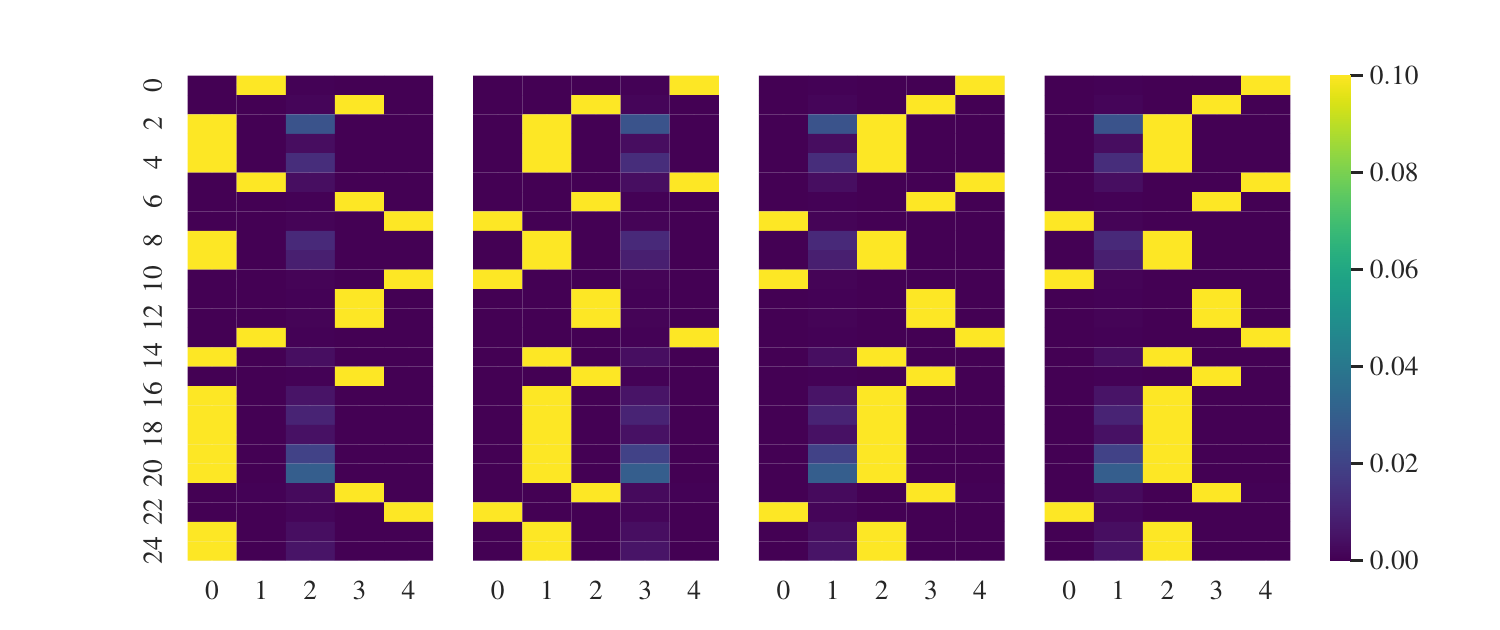}
    \caption{\textbf{Difference between the solutions.} First 25 rows of selected solution $U$ obtained from the global optimality experiment described in \Cref{sec:numerics}.}\label{fig:init_U}
\end{figure*}

\begin{figure*}[!htb]
    \centering
    \includegraphics[width=0.7\textwidth]{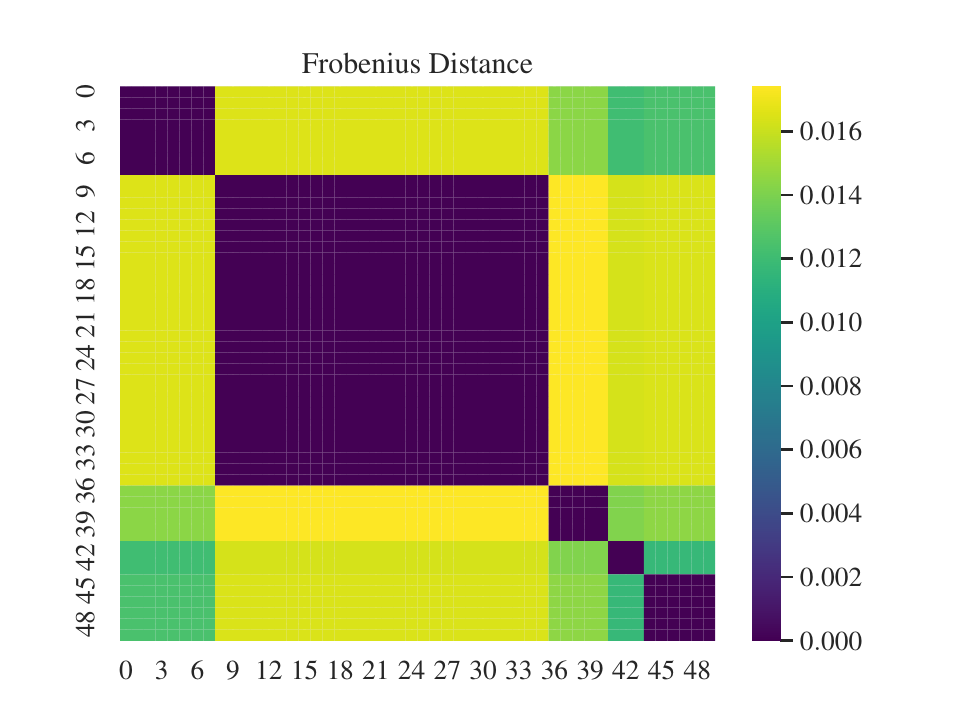}
    \caption{\textbf{Similarities of the membership matrices.} Frobenius distances between the membership matrices $Z$ obtained from the global optimality experiment in \Cref{sec:numerics}, sorted according to their corresponding minimum Hessian eigenvalues.}\label{fig:init_Z}
\end{figure*}

\paragraph{Comparison with another Riemannian clustering method.} Inexact Accelerated Manifold Proximal Gradient Method (I-AManPG) by~\citet{Huang2025} is a recent first-order Riemannian method for solving general problems of the form 
\begin{mini}
{\scriptstyle X}{f(X)+\lambda\norm{X}_1,}{}{}
\addConstraint{X}{\in\mathcal{F}_v}{\coloneqq\{X:X^\top X=I,v\in\operatorname{span}(X)\}.}
\end{mini}
We evaluated its performance on the clustering problem using the on the CyTOF dataset with 50 repetitions (same setting as in~\Cref{fig:cytof_perf}). The results are shown in \Cref{fig:iampg}. While I-AManPG is generally fast and accurate, its median error is higher than that of our methods. In particular, several runs of I-AManPG exhibited large errors, indicating convergence failures.

\begin{figure*}[!htbp]
    \centering
    \includegraphics[width=0.7\textwidth]{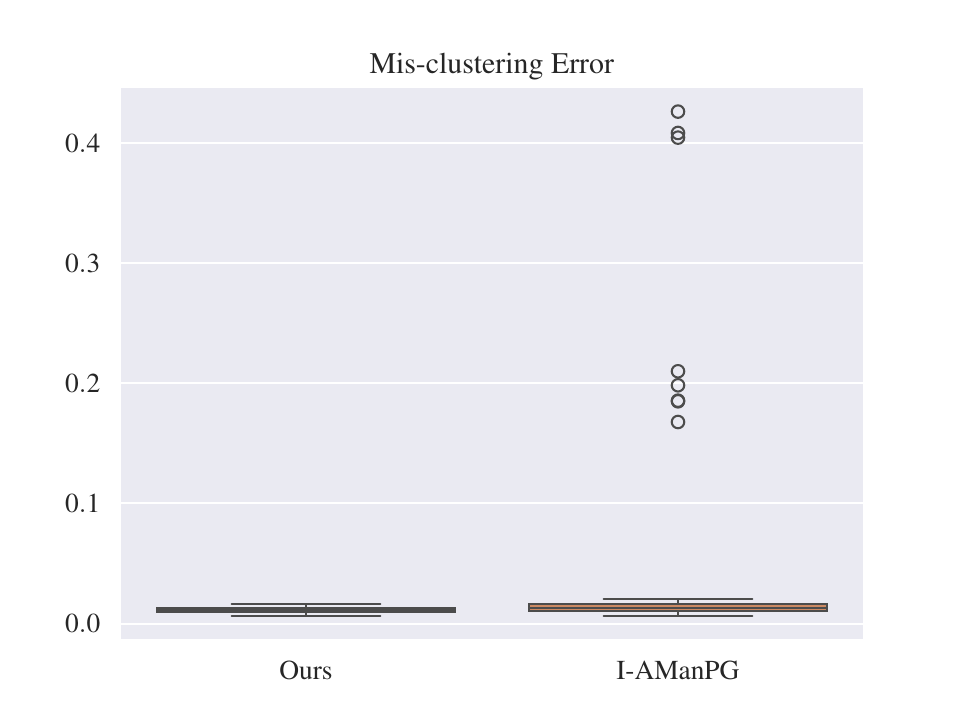}
    \caption{\textbf{Comparison with I-AManPG using CyTOF.} Performance of I-AManPG is comparable to other clustering methods. However, it suffers from convergence failures from time to time and requires careful tuning. Our method again demonstrated its accuracy and stability.}\label{fig:iampg}
\end{figure*}

\paragraph{Hyperparameters tuning.} In the various numerical experiments, we observed that a smaller value of \(\mu\) led to more accurate solutions but at the cost of slower convergence. Therefore, we recommend selecting the largest possible \(\mu\) that does not trigger the phase transition. A good heuristic we found is to choose such that the initial penalty term \(\mu f_2\) remains less than 20 times the main term \(f_1\) in the loss function. The onset of phase transition is also easy to notice, as the algorithm will quickly stagnate and terminate in just a few iterations. If higher accuracy is desired, one can reduce \(\mu\) gradually, using the solution obtained with a larger \(\mu\) as initialization. This warm-start strategy significantly speeds up convergence compared to using a small from the start.

As noted earlier, increasing $r$ can also improve accuracy, likely because it improves the problem landscape. However, due to the $\poly(r,d)$ runtime scaling, caution must be taken when deciding whether to decrease $\mu$ or increase $r$. If tuning $(\mu,r)$ together is desired, we suggest the following strategy: begin by decreasing $\mu$ with $r=K+1$; if this does not produce satisfactory convergence or leads to excessive computation time, try increasing $r$ to $K+2,\dotsc,K+c$ for some small $c$, and repeat the $\mu$-tuning process.

The other hyperparameters in \Cref{alg:riemann_tr} primarily influence the speed of the inner optimization. The initial multiplier $\lambda$ affects only the number of inner steps required during the first iteration. We recommend doing a simple trial run with only two iterations; the resulting optimization history typically offers a reliable guide for choosing an appropriate initial scale for $\lambda$. For the other two parameters, we suggest setting $\kappa_-$ slightly smaller than $\kappa_+$. Empirically, we found $\kappa_-=1.1$ and $\kappa_+=1.3$ work well.

\paragraph{Additional convergence plots.} \Cref{fig:cvg_bhvr_gamma=100}, \Cref{fig:cvg_bhvr_gamma=025}, and \Cref{fig:cvg_bhvr_cytof} illustrate the convergence of our method on GMM with different parameters and on the CyTOF dataset, demonstrating its stability across different datasets.

\begin{figure*}[!htbp]
    \centering
    \includegraphics[width=\textwidth]{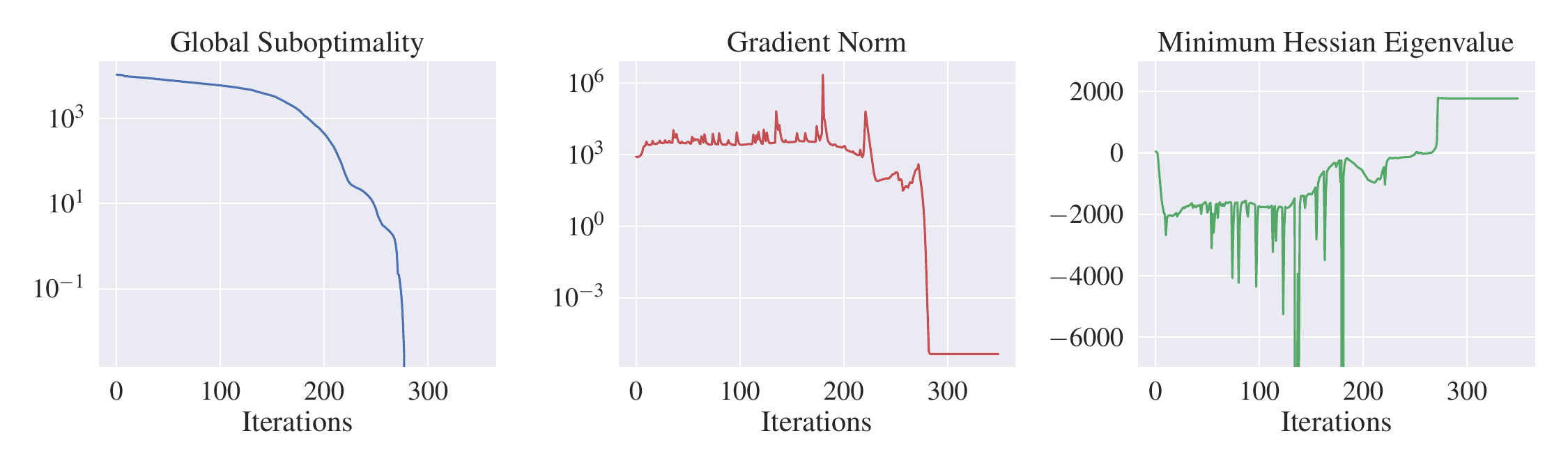}
    \caption{Convergence of our method on GMM data with perfect separation ($n=500,\gamma=1.0$). The loss value steadily decreases over iterations and converges rapidly near the optimal point. This example achieved a perfect final clustering result in the end.}\label{fig:cvg_bhvr_gamma=100}
\end{figure*}

\begin{figure*}[!htbp]
    \centering
    \includegraphics[width=\textwidth]{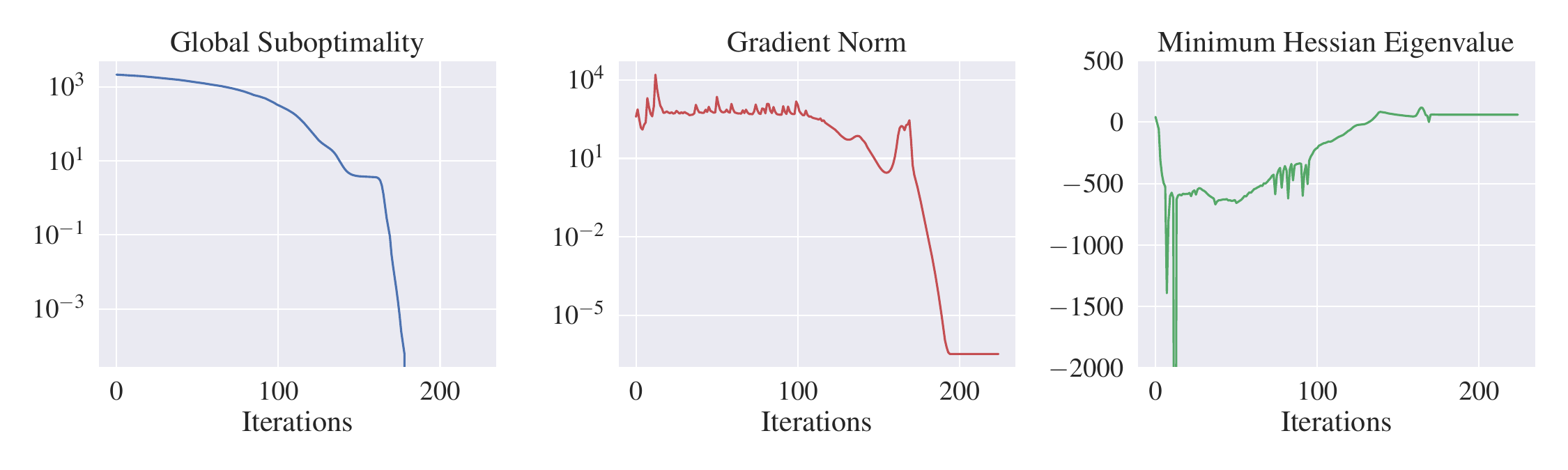}
    \caption{Convergence of our method on synthetic Gaussian mixture data with low separation ($n=500,\gamma=0.25$).}\label{fig:cvg_bhvr_gamma=025}
\end{figure*}

\begin{figure}[!htbp]
    \centering
    \includegraphics[width=\textwidth]{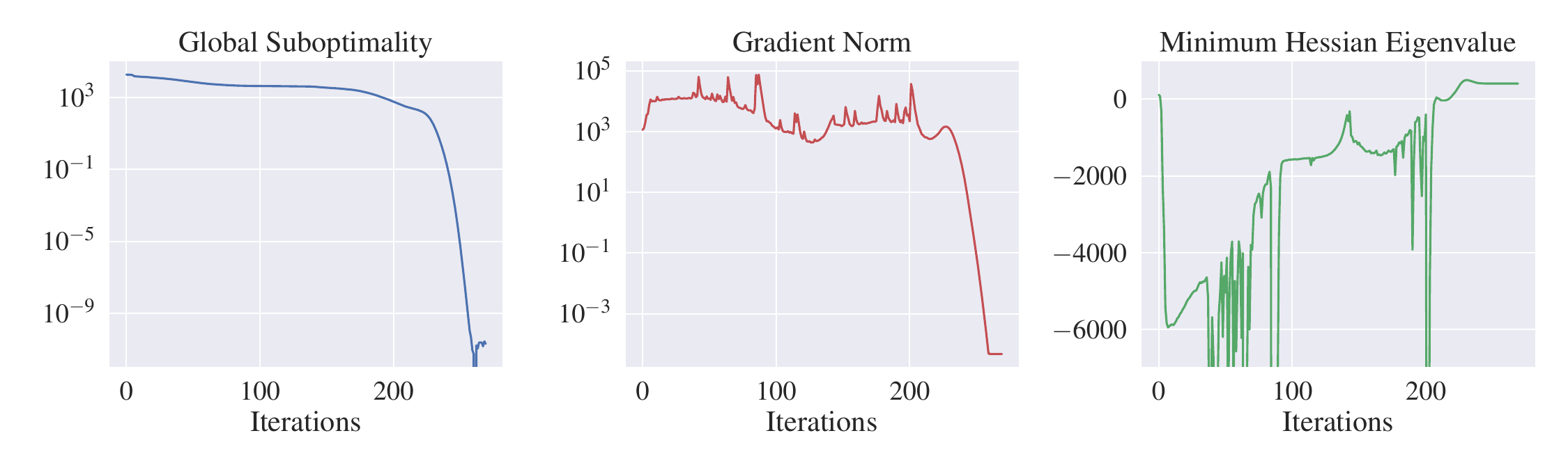}
    \caption{Convergence behavior of our method on the CyTOF dataset.}\label{fig:cvg_bhvr_cytof}
\end{figure}

\paragraph{Hardware information.} All experiments in this work were conducted on a machine equipped with a single Intel Core i9-14900K CPU and 32 GB of RAM.

\paragraph{License of assets.}
The \texttt{MANOPT} solver is distributed under the terms of the GPLv3 license; the \texttt{PYMANOPT} solver is released under the 3-Clause BSD license; the CyTOF dataset is the work of \citet{LEVINE2015184}, cleaned and distributed by \citet{CyTOFClean} under the MIT license.

\paragraph{Code.} A Python demonstration of the proposed method is available at \url{https://github.com/Francis-Hsu/kmeans_manifold}. This repository provides example scripts for the GMM and CyTOF experiments presented in this work.

\section{Pseudocode}\label{appsec:RTR_algo}
This section presents the pseudocode of the proposed method. For its derivation, see~\Cref{appsec:proof_bisect}.
\begin{algorithm}[!htbp]
\caption{Riemannian Second-order Method \label{alg:riemann_tr}}
\begin{algorithmic}[1]
\Require Data \(X\), \\ 
Initial point $(V_0, Q_0)$,\\ 
Initial multiplier $\lambda$,\\
Increment/decrement factors $(\kappa_+,\kappa_-)$,\\
log-barrier penalty $\mu$, \\
Max number of outer/inner iterations $T$ and $B$.
\State  \((V, Q)\gets(V_0, Q_0)\)
\For{$i=1, \dotsc, T$} 
	\State Vectorize the input: \(u\gets[\vect(V^\top),\vect(Q)]^\top\);
    \State Compute the current loss: $\cL=f(u)$;
	\State Compute the Riemannian gradient \(G\gets\nabla f(u)\);
    \State Compute the Jacobian \(J\gets \Diff g(u)\);
	\State Compute \(\tilde{H}=\nabla_{uu}^2\cL(u,y_u)\) as in \Cref{appsec:proof_bisect};
	\For{$j=1, \dotsc, B$}  
		\State \(\dot{v},\dot{q}\gets\Call{SolveInner}{\tilde{H}, G, J,\lambda}\)
		\State Reconstruct \(\dot{V},\dot{Q}\) from vector \(\dot{v},\dot{q}\)
		\State \((V,Q)\gets\Retr_{(V,Q)}(\dot{V},\dot{Q})\)
        \State Compute the new loss $\cL'$ from \((V, Q)\)
		\If{$\cL>\cL'$}
			\State \(\lambda\gets\lambda/\kappa_-\)
            \State \textbf{break}
		\Else
            \State \(\lambda\gets\lambda\cdot\kappa_+\)
		\EndIf
	\EndFor
\EndFor
\State\Output \((V,Q)\)
\State\Function{SolveInner}{$H, G, J, \lambda$}
    \State Add $\lambda I$ to the $V$ and $Q$ blocks of $H$;
    \State Form block matrices as in (\ref{eqn:schur});
    \State$S \gets L_{22}-L_{12}^\top D_{11}^{-1}L_{12}$
    \State$\begin{bmatrix}\dot{q}\\
z
\end{bmatrix}
\gets\begin{bmatrix}b_{2}\\
0
\end{bmatrix}-L_{12}^\top D_{11}^{-1}b_{1}$
    \State$\dot{v}\gets D_{11}^{-1}\mleft(b_{1}-L^{12}\begin{bmatrix}\dot{q}\\
z
\end{bmatrix}\mright)$ 
    \State\Return{\(\dot{v},\dot{q}\)}
\EndFunction

\end{algorithmic}
\end{algorithm}
\end{appendices}

\end{document}